\documentclass[twoside]{article}

\usepackage[accepted]{aistats2025}
\usepackage{natbib}

\usepackage{graphicx,color}
\usepackage{xcolor,colortbl}
\usepackage{amsmath,amssymb}
\usepackage{multicol}
\usepackage{amsthm}
\usepackage{multirow}
\usepackage{booktabs}
\usepackage[skip=0pt]{subcaption}
\usepackage{lipsum}
\usepackage[shortlabels]{enumitem}
\usepackage{cancel}
\usepackage{wrapfig}
\usepackage{array}
\usepackage{siunitx}
\usepackage{csvsimple}
\usepackage[multidot]{grffile}
\usepackage{bbm}
\usepackage{dblfloatfix}
\usepackage{hyperref}
\usepackage{makecell}
\usepackage{bbm, dsfont}
\usepackage{mathtools}
\usepackage{comment}

\usepackage{mathrsfs}
\usepackage{todonotes}
\usepackage{tikz}
\usepackage{hyperref}
\usepackage{cleveref}

\usepackage{algorithm, algorithmic}
\usepackage{xfrac}
\usepackage{thmtools,thm-restate}
\usepackage{graphicx} %
\usepackage{color}

\usepackage{array}
\usepackage{amssymb}
\usepackage{amsmath}
\usepackage{xspace}
\usepackage{fancyhdr}
\usepackage{comment}
\usepackage{bm}

\DeclareMathOperator*{\Exps}{\mathbb{E}}
\DeclareMathOperator*{\argmin}{\mathrm{argmin}}
\newcommand{\ExP}[2]{\Exps_{#1}\left[{#2}\right]}
\newcommand{\Pdata}{\mathbb{P}^{\mathrm{data}}}

\newcommand{\Ptarget}{\mathbb{P}^{\mathrm{target}}}

\newcommand{\br}[1]{\left({#1}\right)}

\DeclareMathOperator{\Trace}{Trace}

\newcommand{\inner}[2]{\left\langle {#1}, {#2} \right\rangle}
\newcommand{\R}{\mathbb{R}}

\newcommand{\bfA}{{\bm{A}}}
\newcommand{\bfB}{{\bm{B}}}

\newcommand{\bfD}{{\bm{D}}}
\newcommand{\bfE}{{\bm{E}}}

\newcommand{\bfI}{{\bm{I}}}

\newcommand{\bfQ}{{\bm{Q}}}

\newcommand{\bfS}{{\bm{S}}}

\newcommand{\bfZ}{{\bm{Z}}}

\newcommand{\bfb}{{\bm{b}}}
\newcommand{\bfc}{{\bm{c}}}

\newcommand{\bfm}{{\bm{m}}}

\newcommand{\bfp}{{\bm{p}}}
\newcommand{\bfq}{{\bm{q}}}
\newcommand{\bfr}{{\bm{r}}}
\newcommand{\bfs}{{\bm{s}}}

\newcommand{\bfw}{{\bm{w}}}
\newcommand{\bfx}{{\bm{x}}}
\newcommand{\bfy}{{\bm{y}}}
\newcommand{\bfz}{{\bm{z}}}
\newcommand{\bfzadv}{\bm{z}^{\textup{adv}}}

\newcommand{\bfSigma}{\ensuremath{\bm{\varSigma}}}
\newcommand{\bftheta}{\ensuremath{\bm{\theta}}}

\newcommand{\bfphi}{\ensuremath{\bm{\phi}}}

\newtheorem{lem}{Lemma}[section]
\newtheorem{remark}{Remark}[section]

\usepackage{arydshln}
\usetikzlibrary{calc}
\usepackage{amsthm}
\usepackage{pifont}% http://ctan.org/pkg/pifont
\newcommand{\xmark}{\ding{55}}%
\newtheorem{theorem}{Theorem}

\newcommand{\maryam}[1]{\textcolor{magenta}{Maryam: #1}}
\newcommand{\djcomment}[1]{\textcolor{brown}{Dj: #1}}
\newcommand{\avi}[1]{\textcolor{blue}{Avi: #1}}

\newcommand{\bmat}[1]{\begin{bmatrix}#1\end{bmatrix}}

\newcommand{\norm}[1]{\|#1\|}
\newcommand{\Thetaspace}{{\boldsymbol{\Theta}}}

\newcommand{\Zspace}{{\boldsymbol{Z}}}
\newcommand{\Aspace}{\mathcal{A}}
\newcommand{\Phispace}{\boldsymbol{\Phi}}

\newcommand{\OPT}{\textup{\sf OPT}}

\begin{document}

% If your paper is accepted and the title of your paper is very long,
% the style will print as headings an error message. Use the following
% command to supply a shorter title of your paper so that it can be
% used as headings.
%
\runningtitle{Keeping up with dynamic attackers: Certifying robustness to adaptive online data poisoning}

% If your paper is accepted and the number of authors is large, the
% style will print as headings an error message. Use the following
% command to supply a shorter version of the authors names so that
% they can be used as headings (for example, use only the surnames)
%
%\runningauthor{Surname 1, Surname 2, Surname 3, ...., Surname n}

\twocolumn[
% \title{Certifying robustness to adaptive data poisoning}
\aistatstitle{Keeping up with dynamic attackers:\\Certifying robustness to adaptive online data poisoning}

\aistatsauthor{Avinandan Bose \And Laurent Lessard \And  Maryam Fazel \And Krishnamurthy Dj Dvijotham}

\aistatsaddress{ University of Washington\\
avibose@cs.washington.edu \And  Northeastern University\\
 l.lessard@northeastern.edu \And University of Washington\\
mfazel@uw.edu \And ServiceNow Research\\ dvij@cs.washington.edu} ]

% \aistatsauthor{%
%   Avinandan Bose
%   % University of Washington\\
%   % avibose@cs.washington.edu\\
%   \And
%   Laurent Lessard
%    % Northeastern University\\
%    % l.lessard@northeastern.edu\\
%    \And
%  Maryam Fazel
%   % University of Washington\\
%   % mfazel@uw.edu\\
%   \And
%   Krishnamurthy Dj Dvijotham
%   % ServiceNow Research\\
%   % dvij@cs.washington.edu\\}
%   }

\begin{abstract}
  The rise of foundation models fine-tuned on human feedback from potentially untrusted users has increased the risk of adversarial data poisoning, necessitating the study of robustness of learning algorithms against such attacks. Existing research on provable certified robustness against data poisoning attacks primarily focuses on certifying robustness for static adversaries who modify a fraction of the dataset used to train the model \emph{before the training algorithm is applied}. In practice, particularly when learning from human feedback in an online sense, adversaries can observe and react to the learning process and inject poisoned samples that optimize adversarial objectives better than when they are restricted to poisoning a static dataset once, before the learning algorithm is applied. Indeed, it has been shown in prior work that online dynamic adversaries can be significantly more powerful than static ones. We present a novel framework for computing certified bounds on the impact of dynamic poisoning, and use these certificates to design robust learning algorithms. We give an illustration of the framework for the mean estimation and binary classification problems and outline directions for extending this in further work. The code to implement our certificates and replicate our results is available at \url{https://github.com/Avinandan22/Certified-Robustness}.\looseness=-1
\end{abstract}

\section{INTRODUCTION}

% \textcolor{red}{Avi and Dj, please proof read. There've been a lot of typos throughout even in the abstract...}

With the advent of foundation models fine tuned using human feedback gathered from potentially untrusted users (for example, users of a publicly available language model) \citep{christiano2017deep, ouyang2022training}, the potential for adversarial or malicious data entering the training data of a model increases substantially. This motivates the study of robustness of learning algorithms to poisoning attacks \citep{biggio2012poisoning}. More recently, there have been works that attempt to achieve ``certified robustness`` to data poisoning, i.e., proving that the worst case impact of poisoning is below a certain bound that depends on parameters of the learning algorithm. All the work in this space, to the best of our knowledge, focuses on the \emph{static} poisoning adversary \citep{steinhardt2017certified, zhang2022bagflip}. Even in \citep{wang2024temporal} which is the closest setting to our work, the poisoning adversary acts over offline datasets in a temporally extended fashion which are poisoned in one shot, and thus is not %purely
dynamic.
% \textcolor{magenta}{this sentence is not clear to me, given it's the closest paper we should explain more}. \avi{Changed it. The paper really isn't dynamic is the main message.}
% \textcolor{magenta}{it is still not clear to me what the paper does though, I'll have to look at it. Also what we  mean by static/dynamic dataset and static/dynamic adversary doesn't come across in the intro. We should point to concrete definition later in the paper.}
%
There has been work on \emph{dynamic} attack algorithms \citep{zhang2020online, wang2018data} showing that these attacks can indeed be more powerful than static attacks. %adversaries.
This motivates the question we study: can we obtain certificates of robustness for a broad class of learning algorithms against \emph{dynamic} poisoning adversaries?

In this paper, we study learning algorithms corrupted by a dynamic poisoning adversary who can observe the behavior of the algorithm and adapt the poisoning in response. This is relevant in scenarios where models are continuously/periodically updated in the face of new feedback, as is common in RLHF/fine tuning applications (see Figure~\ref{fig:schematic}). We provide (to the best of our knowledge) the first general framework for computing certified bounds on the worst case impact of a dynamic data poisoning attacker, and further, use this certificate to design robust learning algorithms (see Section~\ref{sec:setup}). We give an illustration of the framework for the mean estimation problem (see Section~\ref{sec:mean_estimation}) and binary classification problem (see Section~\ref{sec:binary_classification}) and suggest directions for future work to apply the framework to
%more realistic
other practical learning scenarios.
Our contributions are as follows:
\begin{enumerate}
    \item We develop a
    %generally applicable
    % general optimization-based
    framework for computing certified bounds on the worst case impact of a dynamic online poisoning adversary on a learning algorithm as a finite dimensional optimization problem. The framework applies to an arbitrary learning algorithm and a general adversarial formulation described in Section~\ref{sec:setup}. However, instantiating the framework in a computationally tractable way requires additional work, and we show that this instantiation can be done for certain cases.
    %, leaving extensions to broader learning algorithms to future work.
    \item We demonstrate that for learning algorithms designed for mean estimation (Section~\ref{sec:mean_estimation}) and binary classification problems with linear classifiers (Section~\ref{sec:binary_classification}), we can tractably compute bounds (via dual certificates) for learning algorithms that use either regularization or noise addition as a defense against data poisoning. We leave extensions to broader learning algorithms to future work.    \looseness=-1
    \item We use these %derived
    certificates to choose parameters of a learning algorithm so as to trade off performance and robustness, and thereby derive robust learning algorithms (Section~\ref{sec:meta_learning}).
    \item We conduct experiments on real and synthetic datasets to empirically validate our certificates of robustness, as well as using the meta-learning setup to design defenses (Section~\ref{sec:experiments}).
\end{enumerate}

\begin{figure*}
    \centering
    \includegraphics[width=0.8\linewidth]{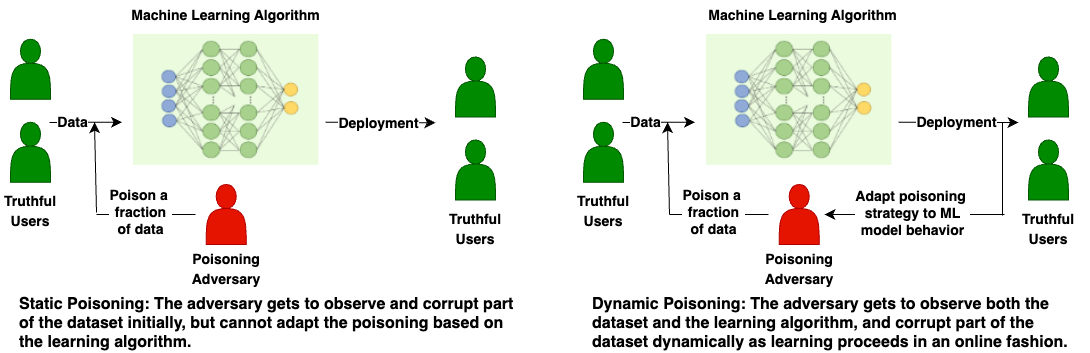}
    \caption{A schematic diagram to highlight the differences between static and dynamic poisoning. }
    \vskip -0.2in
    \label{fig:schematic}
\end{figure*}
\subsection{Related Work}\label{appendix:related_work}
\paragraph{Data Poisoning.} Modern machine learning pipelines involve training on massive, uncurated
datasets that are potentially untrustworthy and of such scale that conducting rigorous quality checks becomes impractical. Poisoning attacks \citep{biggio2012poisoning, newsome2006paragraph, biggio2018wild} pose big security concerns upon deployment of ML models. Depending on which stage (training / deployment)  the poisoning takes place, they can be characterised as follows: 1. Static attacks: The model is trained on an offline dataset with poisoned data. Attacks could be untargeted, which aim to prevent training convergence rendering an unusable model and thus denial of service \citep{tian2022comprehensive}, or targeted, which are more task-specific and instead of simply increasing loss, attacks of this kind seeks to make the model output wrong predictions on specific tasks. 2. Backdoor attacks: In this setting, the test / deployment time data can be altered \citep{chen2017targeted, gu2017badnets, han2022physical, zhu2019transferable}. Attackers manipulate a small proportion of the data such that, when a specific pattern / trigger is seen at test-time, the model returns a specific, erroneous prediction. 3. Dynamic (and adaptive) attacks: In scenarios where models are continuously/periodically updated in the face of new feedback, as is common in RLHF/fine tuning applications, a dynamic poisoning adversary \cite{wang2018data, zhang2020online} can observe the behavior of the learning algorithm and adapt the poisoning in response.
% \paragraph{Dynamic Adversaries}
\paragraph{Certified Poisoning Defense.} Recently, there have been works that attempt to achieve ``certified robustness" to data poisoning, i.e., proving that the worst case impact of \textit{any} poisoning strategy is below a certain bound that depends on parameters of the learning algorithm. All the work in this space, to the best of our knowledge, focuses on the \emph{static} or \textit{backdoor} attack adversary. \citep{steinhardt2017certified} provide certificates for linear models trained with gradient descent, \citep{rosenfeld2020certified}  present a statistical upper-bound on the effectiveness of $\ell_2$ perturbations on training labels for linear models using randomized smoothing, \citep{zhang2022bagflip, sosnin2024certified} present a model-agnostic certified approach that can effectively defend against both trigger-less and backdoor attacks, \citep{xie2022uncovering} observe that differential privacy, which usually covers addition or removal of data points, can also provide statistical guarantees in some limited poisoning settings. Even in \citep{wang2024temporal} which is the closest setting to our work, the poisoning adversary acts over offline datasets in a temporally extended fashion which are poisoned in one shot, and thus is not dynamic.
\section{PROBLEM SETUP}\label{sec:setup}
% \avi{Editing this section}
We now develop the exact problem setup that we study in the paper. We consider a learning algorithm aimed at estimating parameters $\bftheta \in \Theta$, and each step of the learning algorithm
updates the estimates of these parameters based on potentially poisoned data. The following components fully define the problem setup. A notation table in provided in Appendix~\ref{appendix:notation}.
% \textcolor{red}{In table 1: instead of learning algorithm, say something like "updates of learning algorithm". Also, what does "parameters of model" on the first line mean? what model? think of a term that makes the reader get what you refer to quickly. \\
% In table 2: add border to table. also explain more in the paper text what the first two columns mean, e.g., what does 'for deployed model' imply, that the model is fixed? or that the adversary can't observe behavior? or can't react to behavior?}\\
% \avi{Table 1 : I don't see why model should be confusing? Very clearly we are working with ML models / foundational models. I'll let DJ respond since he made the table.}\\
% \avi{Can you please suggest changes? A potential extra column to incorporate your doubts will be ''periodic updates", I'm not aware of literature considering periodic updates in poisoning.} \textcolor{red}{doesn't need to be in the table, can explain more in the text.}

% \textcolor{red}{can also add $z_{\rm adv}\in\mathcal{A}$ to notation table}

\begin{table*}[t]
\centerline{
% \adjustbox{maxwidth=\linewidth}{
% \adjustbox{
% \notsotinyfortable{
  \begin{tabular}{l|c|c|c}
  \multirow{3}{*}{Attack Type} & Adversary adapts poisoning  & Adversary can   & \multirow{2}{*}{Certified}\\
  & strategy upon observing & poison data  & \multirow{2}{*}{robustness} \\
   & model behavior & for deployed model &   \\
  \hline
   Static / One-shot
   % (\citep{tian2022comprehensive, steinhardt2017certified, rosenfeld2020certified})
   &  \xmark\cellcolor{red!25} &\xmark\cellcolor{red!25} & \checkmark\cellcolor{green!25} \\
  \hline
   Backdoor
   % (\citep{chen2017targeted, gu2017badnets, han2022physical, zhu2019transferable, zhang2022bagflip, sosnin2024certified})
   & \xmark\cellcolor{red!25} & \checkmark\cellcolor{green!25}  &\checkmark\cellcolor{green!25} \\
   \hline
   Dynamic attack only
   % (\cite{wang2018data, zhang2020online}
   &  \checkmark\cellcolor{green!25} & \checkmark\cellcolor{green!25}  &\xmark\cellcolor{red!25} \\
  \hline
    Dynamic attack \& defense (Ours) & \checkmark\cellcolor{green!25}  & \checkmark\cellcolor{green!25} & \checkmark\cellcolor{green!25} \\
\end{tabular}
% }
}
\caption{A comparison with lines of work closest to ours. Static/One-shot (\citep{tian2022comprehensive, steinhardt2017certified, rosenfeld2020certified}), Backdoor(\citep{chen2017targeted, gu2017badnets, han2022physical, zhu2019transferable, zhang2022bagflip, sosnin2024certified}), Dynamic attack only (\cite{wang2018data, zhang2020online}.A detailed description is provided in Section~\ref{appendix:related_work}.}
\label{tab:my_label}
\vskip -0.2in
\end{table*}

\paragraph{Online learning algorithm.}
%\maryam{I changed to 'parameters' to 'model parameters' below to be clear at least the first time it appears in the text. This also agrees with the table.}
% \avi{Is the word hyperparameter appropriate when trying to capture general learning algorithms? What are some other suggestions?}\djcomment{I think it is a good enough word}
We consider online learning algorithms that operate by receiving a new datapoint at each step  and making an update to model parameters being estimated. In particular, we consider learning algorithms that can be written as
% \begin{align}\bftheta_{t+1} \gets F\br{\underbrace{\bftheta_t}_{\text{Parameter estimate at time $t$}},  \underbrace{\bfz_t}_{\text{Datapoint received at time $t$}},\underbrace{\bfphi_t}_{\text{Hyperparameters at time $t$
% }}}\label{eq:dyn}\end{align}
\begin{align}\bftheta_{t+1} \gets F_{\phi}\Bigl(\underbrace{\bftheta_t}_{\text{Parameter}}, \underbrace{\bfz_t}_{\text{Datapoint}}\Bigr),\label{eq:dyn}
\end{align}
where $F_{\phi}: \Thetaspace \times \Zspace \to \Thetaspace$ is a parameterized function that maps the current model parameters $\bftheta_t$ to new model parameters $\bftheta_{t+1}$, based on the received datapoint $\bfz_t$, where $\phi \in \Phispace$ is a hyperparameter, for example, learning rate in a gradient based learning algorithm, or strength of regularization used in the objective function.

% \maryam{Maybe we can move the notation table to the appendix and add a sentence here that refers to it in appendix.}

\textbf{Example} To illustrate the setup we consider a simple toy example where we try to estimate the mean of the datapoints via gradient descent on the $\ell_2$ regularized squared Euclidean loss. Given a current estimate $\bftheta_t$, upon receiving a datapoint $\bfz_t$, the update step can be written as:
\begin{align*}
    \bftheta_{t+1} &= \bftheta_t - \eta \nabla \bigl(\tfrac{1}{2}\|\bfz_t - \bftheta_t\|_2^2 + \tfrac{\sigma}{2} \|\bftheta_t\|_2^2\bigr)\\
    &= (1 - \eta - \eta \sigma) \bftheta_t + \eta \bfz_t\\
    &= F_{\phi}(\bftheta_t, \bfz_t).
\end{align*}
Here $\phi = \{\eta, \sigma\}$ denotes the learning rate and regularization parameter, and are the hyperparameters of the learning algorithm.
% \djcomment{Maybe provide a few examples of $F$ here, clarifying what $\theta, \phi, ...$ correspond to?}
%
Note that $F_\phi$ is a general update rule and we do not make any assumptions about $F_\phi$.

% \djcomment{Commented out what was below, it seemed unnecessarily confusing. I think what would provide more clarity is to just include a few examples of update rules (gradient descent, momentum, maybe even a zero-th order rule).}
%The hyperparameters $\phi_t$ can include a variety of elements, such as the step size in gradient-based algorithms (which may vary over time, as is the case with several popular gradient methods, such as SGD, AdaGrad, Adam etc.), regularization parameters of the model, or even additive Gaussian noise (making $\phi_t$ potentially stochastic). The distribution of hyperparameters at time $t$ is denoted by $\mathbb{P}_t^\phi$. \avi{Is the message that the update rule is very general coming across? Not having the time dependence can make some of the writing much easier -- and in both the examples we have considered, we draw certificates based on only the stationary state.}

% $\phi_t$ indicates model hyperparameters at time $t$ which can denote (but are not limited to) learning rates, regularization parameters, additive gaussian noise.
% The exogenous noise input refers to noise artificially injected into the training algorithm in order to make the algorithm more robust to potential poisoning.
% We further assume that the distribution of $\bfw_t$ is independent of $t$ and each $\bfw_t$ is sampled iid.

\paragraph{Poisoned learning algorithm.}
% We work in a setting where the datapoints received by the learning algorithm may be corrupted by an adversary, with the corruption allowed to be a function of the entire trajectory of the learning algorithm up to
% % \textcolor{red}{remember to fix 'up to -> up to' everyhwere}
% that point. Whether or not the poisoned datapoint is picked is chosen probabilistically \textcolor{red}{do we use two probabilities, one for data point getting poisoned and one for it being picked? can we merge these into one?} \avi{They're already the same. Poisoning doesn't mean picking a benign point and modifying it, but rather introducing a corrupted data point that is within a typical set} \textcolor{red}{so then let's just write it in a simpler way with one probability}, and we assume a fixed probability that the poisoned datapoint is chosen over a clean datapoint \textcolor{red}{is this model commonly used? can we cite sources? (for model with separate poisoning and selection)}\avi{DJ?}.

% We work in a setting where some of the data points received by the learning algorithm are corrupted by an adversary, with the corruption allowed to be a function of the entire trajectory of the learning algorithm up to that point. We refer to an adversary of this type who can observe the entire trajectory upto that point to select a corruption for the next step as dynamic adversary.
We consider a setting where some of the data points received by the learning algorithm are corrupted by an adversary, who is allowed to choose corruptions as a function of the entire trajectory of the learning algorithm up to that point. We refer to such an adversary, who can observe the full trajectory and decide on the next corruption accordingly, as a \textit{dynamic adaptive adversary}. While this may seem unrealistic, since our goal here is to compute certified bounds on the worst case adversary, we refrain from placing informational constraints on the adversary, as an adversary with sufficient side knowledge can still infer hidden parameters of the model from even from just a prediction API \cite{tramer2016stealing}.

The adversary is restricted to select a corrupted data point $\bfzadv_t \in \Aspace$,  which reflects constraints such as input feature normalization or the adversary trying to avoid outlier detection mechanisms used by the learner. We make no additional assumptions about the specific poisoning strategy employed by the adversary.
Thus, our certificates of robustness to poisoning apply to \emph{any dynamic adaptive adversary who chooses poisoned data points from the set $\Aspace$.}

We assume that with a fixed probability, the data point the algorithm receives at each time step is poisoned. In practice, this could reflect the situation that out of a large population of human users providing feedback to a learning system, a small fraction are adversarial and will provide poisoned feedback. Let $\Pdata$ denote the benign distribution of data points. Mathematically, the data point $\bfz_t$
received by the learning algorithm at time $t$ is sampled according to $\bfz_t \sim  \epsilon \delta
%\text{Dirac}
(\bfzadv_t) + (1-\epsilon)\Pdata$,
where $\delta(\cdot)$ denotes the Dirac delta function,
and $\epsilon$ is a parameter that controls the ``level'' of poisoning (analogous to the fraction of poisoned samples in static poisoning settings \citep{steinhardt2017certified}). This is a special case of Huber's contamination model, which is used in the robust statistics literature \citep{diakonikolas2023algorithmic} with the contamination model being a Dirac distribution. For compactness of the data generation process we define the following:
\begin{align}
    \mathbb{P}_{\epsilon}(\bfzadv):= \epsilon\delta
%\text{Dirac}
(\bfzadv) + (1-\epsilon)\Pdata.
\end{align}

% \maryam{how about using the notation $\delta$ instead of 'Dirac'? it's shorter and more common in probability and in engineering.}

%\paragraph{Potential Defense} Inspired by differentially private learning algorithms like DP-SGD \citep{bassily2014private}, we propose adding Gaussian noise to the learning process as a way of smoothing the learning algorithm against impacts of the poisoning adversary. In particular, we add $\bfB\bfw_t$ where $\bfw_t$ is iid noise in each step sampled from the standard Gaussian, and $\bfB$ is a design parameter of the learning algorithm. Subsequently, we will choose $\bfB$ so as to minimize the worst case impact of the poisoning adversary. We denote by $\bfS=\bfB\tran{\bfB}$ the covariance matrix of the noise added.
% More generally, we consider learning algorithms that can be expressed as $\bftheta_{t+1} \gets F_\phi\br{\nabla \ell\br{\bftheta_t, \bfz_t}, \bfw_t}$ where $\bfw_t$ is noise and $\phi$ are tunable parameters of the learning algorithm.

\subsection{Adversarial Objective}
\paragraph{Transition Kernel.}
%\maryam{It's worth writing this sentence more clearly about what order things happen in}
Starting with a parameter estimate $\bftheta$, the adversary chooses $\bfzadv$, then the learning algorithm updates the parameter estimate (via $F_\phi$). The transition kernel gives the probability (or probability density) that the parameter estimate assumes a value $\bftheta^\prime$ after the above steps, and is defined as (recall that $\delta(\cdot)$ denotes the Dirac delta function):
\begin{align}
    \Pi(\bftheta^\prime | \bftheta, \bfzadv) = \ExP{\bfz \sim \mathbb{P}_{\epsilon}(\bfzadv)}{\delta
    %\text{Dirac}
    (F_{\phi}(\bftheta, \bfz) - \bftheta^\prime)}.
\end{align}
\paragraph{Dynamics as a Markov Process.}
The dynamics in Eq. \eqref{eq:dyn} gives rise to a Markov %chain
process over the parameters $\bftheta$. If $\mathbb{P}_t$ denotes the distribution over parameters at time $t$, we have
% \[
% \mathbb{P}_{t+1}\br{\bftheta^\prime} = \int \mathbb{P}_{F, \Pdata, \mathbb{P}^{\bfphi}_t}\br{\bftheta^\prime|\bftheta, \bfzadv_t}
% \mathbb{P}_{t}\br{\bftheta}d\bftheta,
% \]
\begin{align}
    \mathbb{P}_{t+1}\br{\bftheta^\prime} = \int \Pi(\bftheta^\prime | \bftheta, \bfzadv)
\mathbb{P}_{t}\br{\bftheta}d\bftheta.
\end{align}
Since the learning algorithm (dynamics of the parameters) is a Markov process, the sequence of actions for the adversary (i.e., choices of $\bfzadv$) constitute a Markov Decision Process with
% \begin{align*}
%  \underbrace{\bftheta}_{\text{States}},  \underbrace{\bfzadv}_{\text{Actions}},   \underbrace{\mathbb{P}_{F, \Pdata, \mathbb{P}^{\bfphi}_t}\br{\bftheta^\prime|\bftheta, \bfzadv}}_{\text{Transition Kernel}}.
% \end{align*}
\begin{align*}
 \underbrace{\bftheta}_{\text{States}},  \underbrace{\bfzadv}_{\text{Actions}},   \underbrace{\Pi(\cdot | \bftheta, \bfzadv) }_{\text{Transition Kernel}}.
\end{align*}
\paragraph{Adversarial objective function.}

We assume that the poisoning adversary is interested in maximizing some adversarial objective $\ell_{\textrm{adv}}: \Thetaspace \mapsto \R$, for example, the expected prediction error on some target distribution of interest to the adversary. The adversary wants to choose actions such that it can maximize its average reward over time. Utilizing the fact that the optimal policy for an MDP is stationary (i.e., the policy is time invariant), we define the adversary's objective for an arbitrary stationary policy $\bfzadv \sim \pi(\cdot| \bftheta)$ as follows:
\begin{align}
    \rho(\pi) = \lim_{T \to \infty}\frac{1}{T} \ExP{\pi}{\sum_{t=1}^T \ell_{\textrm{adv}} (\bftheta_t)}, \label{eq:avg_cost}
\end{align}
where the expectation is with respect to the noisy state transition dynamics induced by the adversary's poisoning policy $\pi$.

We utilize the fact that $\rho(\pi)$ is equal to the expected reward under the \textit{stationary state distribution}
% under policy $\pi$, denoted by $d_{\pi}$
(assuming the MDP is ergodic, see details in Appendix~\ref{appendix:proofs}):
% \maryam{should we add 'if it exists'? did we discuss this in the appendix? if not we should write the appropriate assumption (ergodic MDP/ stable dynamics?)} under policy $\pi$  denoted by $d_{\pi}$,
\begin{align*}
    \rho(\pi) = \ExP{\bftheta \sim d_{\pi}(\cdot)}{\ell_{\rm adv}(\bftheta)}.
\end{align*}
The stationary state is defined as a condition where the distribution of parameters remains unchanged over time. In other words, the distribution of parameters at any given point in the stationary state is identical to the distribution of parameters at the next state.

The stationarity condition can be expressed mathematically in terms of the transition kernel as:
\begin{align*}
    \ExP{\substack{\bftheta \sim d_{\pi}(\cdot) \\ \bfzadv \sim \pi(\cdot|\bftheta)}}{\Pi(\bftheta^\prime | \bftheta, \bfzadv)} =  d_{\pi}(\bftheta^\prime) \quad \forall \bftheta^\prime \in \Theta.
\end{align*}

Given a family of learning algorithms $F_\phi$ with tunable parameters $\phi \in \Phispace$, our goal is to estimate $\phi$ so that our learning algorithm is robust to the poisoning as described above. However, since we assume that we are working in the online setting, it is seldom the case that we know the data distribution $\Pdata$ in advance, making the adversary's objective intractable. In Section~\ref{sec:meta_learning}, we use a meta learning formulation to overcome the lack of knowledge about $\Pdata$ in advance.

\subsection{Meta-learning a robust learning algorithm}\label{sec:meta_learning}
% The above certificate holds for a fixed data distribution.
 In a meta learning setup \citep{hochreiter2001learning, andrychowicz2016learning}, we suppose that we have access to a meta-distribution from which data distributions can be sampled. In such a setup, we can ``simulate" various data distributions and consider the following approach: We take a family of learning algorithms $F_\phi$ with tunable parameters $\phi \in \Phispace$,
 % Based on the results from Section \ref{sec:certificate}, we can
 and attempt to design the parameters $\phi$ of the learning algorithm to trade-off performance and robustness in expectation over the data distributions sampled from the meta-distribution.
% In particular, in the absence of poisoned data, assume that the updates \eqref{eq:dyn} result in a stationary distribution $\mathbb{P}\br{\phi, \Pdata}$ over model parameters $\bftheta$.

In particular, in the absence of poisoned data, the updates \eqref{eq:dyn} on data sampled from benign data distribution $\Pdata$ result in a stationary distribution over model parameters $\bftheta$ denoted by $\mathbb{P}\bigl(\phi, \Pdata\bigr)$. The expected benign target loss can be written as :
\begin{align}
    b(\phi, \Pdata) = \ExP{\bftheta \sim \mathbb{P}\br{\phi, \Pdata}}{\ell\br{\bftheta}},
\end{align}
where $\ell : \Thetaspace \mapsto \R$ is the loss the learning algorithm wants to minimize.
% We explicitly denote the transition kernel, in case of a posioned learning algorithm receving benign samples from $\Pdata$.

In Section~\ref{sec:certificate}, we propose a general formulation to derive an upper bound on the worst case impact of an adversary on the target loss (a certificate), which we denote by $c(\phi, \Pdata)$ for a given data distribution $\Pdata$ and parameter of the learning algorithm $\phi \in \Phispace$.

Given a meta distribution $\mathcal{P}$, we can propose the following criterion to design a robust learning algorithm:
% \avi{I think the infimum over $\lambda$ should be before the outer expectation, since we are getting the certificate per learner. }\djcomment{Fair enough, I am a bit skeptical that we would actually be able to solve the version with an inner minimization over $\lambda$ as that would require a lot of computation/memory (we would need to store one $\lambda$ for each data distribution). In particular, think about some extremely large source of data distributions and we can just about do one pass of stochastic gradients in the outer loop, trying to solve for the optimal $\lambda$ for each will probably be too much. Solving for a common certificate $\lambda$ will be computationally cheaper and hopefully more generalizable.}
% \begin{align}\label{eq:main_formulation}
%     &\inf_{\substack{\phi \in \Phispace \\ \lambda: \Thetaspace \mapsto \R}} \mathbb{E}_{\Pdata \sim \mathcal{P}}\bigg[\ExP{\bftheta \sim \mathbb{P}\br{\phi, \Pdata}}{\ell\br{\bftheta}}
%      \\&\; +\kappa \Bigl(\!\sup_{\substack{\bftheta \in \Thetaspace\\\bfzadv \in \Aspace}} \ExP{\bftheta^\prime \sim \Pi(\cdot| \bftheta, \bfzadv)}{\lambda\br{\bftheta^\prime}} + \ell_{\textup{adv}}\br{\bftheta}-\lambda\br{\bftheta}\Bigr)\bigg], \nonumber
% \end{align}
\begin{align}\label{eq:main_formulation}
    &\inf_{\phi \in \Phispace} \mathbb{E}_{\Pdata \sim \mathcal{P}}\bigg[b(\phi, \Pdata) + \kappa \cdot c(\phi, \Pdata) \bigg].
     % \\&\; +\kappa \Bigl(\!\sup_{\substack{\bftheta \in \Thetaspace\\\bfzadv \in \Aspace}} \ExP{\bftheta^\prime \sim \Pi(\cdot| \bftheta, \bfzadv)}{\lambda\br{\bftheta^\prime}} + \ell_{\textup{adv}}\br{\bftheta}-\lambda\br{\bftheta}\Bigr)\bigg], \nonumber
\end{align}
% \maryam{what is the subscript $S$ in the inner problem above? It didn't apear earlier, so let's define here}
where $\kappa > 0$ is a trade-off parameter. The expectation over $\mathcal{P}$ is a meta-learning inspired formulation, where we are designing a learning algorithm that is good ``in expectation" under a meta-distribution over data distributions. The first term constitutes ``doing well" in the absence of the adversary by converging to a stationary distribution over parameters that incurs low expected loss. The second term is an upper bound on the worst case loss incurred by the learning algorithm in the presence of the adversary.\looseness=-1

% \begin{algorithm}[t]
%     \caption{Meta learning a robust learning algorithm}
%     \begin{algorithmic}[1]
%         \STATE \textbf{Input: } Set of $K$ distributions sampled from $\mathcal{P}[\Pdata]$, tradeoff parameter $\kappa$.
%         \STATE For an arbitrary $\phi \in \Phispace$ write the expression of the following quantities.
%         \FOR{$\Pdata \in \mathcal{P}$}
%         \STATE \textbf{Benign Loss:} $b_i(\phi) = \ExP{\bftheta \sim \mathbb{P}\br{\phi, \Pdata}}{\ell\br{\bftheta}}$
%         \STATE \textbf{Certificate of Robustness:}
%         $c_i(\phi) = \inf_{\lambda: \Thetaspace \mapsto \R} \sup_{\substack{\bftheta \in \Thetaspace \\ \bfzadv \in \Aspace}} \ExP{\bftheta^\prime \sim \Pi(\cdot | \bftheta, \bfzadv)}{\lambda\br{\bftheta^\prime}} + \ell_{\textup{adv}}\br{\bftheta}-\lambda\br{\bftheta}$
%          \ENDFOR
%         \STATE $\widehat{\phi} = \inf_{\phi \in \Phi} \sum_{i \in [K]} \left[b_i(\phi) + \kappa c_i(\phi)\right]$.
%         % \IF{$\bfS^{(t+1)} = \bfS^{(t)}$}
%         % \STATE break
%         % \ENDIF
%         \RETURN $\widehat{\phi}$
%     \end{algorithmic}
%     \label{alg:meta_learning}
% \end{algorithm}

\subsection{Technical Approach: Certificate of Robustness} \label{sec:certificate}

For a given $\Pdata \sim \mathcal{P}$, we attempt to find an upper bound on the worst case impact of the adversary.
Recalling that the sequence of actions for the adversary constitutes a Markov Decision Process, the value of the adversarial objective for the adversary's optimal action sequence is therefore the average reward in the infinite horizon Markov Decision Process setting \citep{malek2014linear} and can be written as the solution of an \emph{infinite dimensional} linear program (LP)  \cite{puterman2014markov}.
% The adversary is interested in identifying a
% In particular, for the infinite horizon average reward setting .
The LP can be written as:
\begin{align}\label{eq:lp_inf}
    &\sup_{\substack{d_{\pi} \in \mathcal{P}[\Thetaspace] \\ \pi \in \mathcal{P}[\Thetaspace \times \bfZ]}}  \quad \ExP{\bftheta \sim d_{\pi}(\cdot)}{\ell_{\textup{adv}}\br{\bftheta}}, \quad \text{subject to } \\
       & \ExP{\substack{\bftheta \sim d_{\pi}(\cdot) \\ \bfzadv \sim \pi(\cdot|\bftheta)}}{\Pi(\bftheta^\prime | \bftheta, \bfzadv)} =  d_{\pi}(\bftheta^\prime) \quad \forall \bftheta^\prime \in \Theta \nonumber,
\end{align}
% \djcomment{While it is nice to have the stationarity condition, I don't think it is necessary here. We should first state the problem we are trying to solve - at stationarity, what is the expected rewards of the attacker? Then theorem 1 provides a way to bound this quantity.}
% \end{subequations}
% \avi{I rewrote the set of equations below: check if it is notationally more consistent that the one above. Maybe not much of a difference on second thought.}
% In particular, for the infinite horizon average reward setting \citep{malek2014linear}, the LP can be written as
% \begin{subequations}
% \begin{align}
%     \sup_{\mu} & \ExP{\bftheta, \bfzadv \sim \mu}{\ExP{\bftheta^\prime | \bftheta, \bfzadv}{\ell_{\textup{adv}}\br{\bftheta^\prime}}} \\
%     \text{Subject to } & \mu \in \mathcal{P}[\R^d \times \R^n] \\
%     & \ExP{\bftheta, \bfzadv \sim \mu}{\mathbb{P}_{\bfS, \Pdata, \bfzadv}\br{\bftheta^\prime|\bftheta}} = \ExP{\bftheta, \bfzadv \sim \mu}{\mathbb{I}[\bftheta^\prime=\bftheta]} \quad \forall \bftheta^\prime \in \R^d
% \end{align}
% \end{subequations}
% \maryam{the optimization problem should be referred to with a single equation number, i.e., problem (5) in thm 1 (not 5a)}
where $\mathcal{P}[\Thetaspace]$, $\mathcal{P}[\Thetaspace \times \bfZ]$ denote the space of probability measures on $\Thetaspace$ and $\Thetaspace \times \bfZ$ respectively.
% and $\mathbb{I}$ denotes the indicator function that equals $1$ if its argument is true and $0$ otherwise.
%\maryam{Dj, are there other papers to cite that use this formulation?}
%\djcomment{Maryam, we quote references [13] and [17] based on which this is derived. Do we need anything further? If you mean specifically in the context of poisoning, I am pretty sure not, no other papers use this.}

We are now ready to present our %technical result, a
certificate of robustness against dynamic data poisoning adversaries, which is the largest objective value any dynamic adversary can attain in the stationary state.
% \begin{theorem}\label{thm:certificate}
\begin{restatable}{theorem}{certificate}\label{thm:certificate}
 For any function $\lambda: \Theta \mapsto \R$, for any dynamic adaptive adversary, the average loss \eqref{eq:avg_cost} is bounded above by
% \begin{align}
%     \sup_{\substack{\bftheta \in \Thetaspace \\ \bfzadv \in \Aspace}} \qquad \ExP{\bftheta^\prime \sim \mathbb{P}_{F, \Pdata, \bfzadv}\br{\cdot|\bftheta}}{\lambda\br{\bftheta^\prime}} + \ell_{\textup{adv}}\br{\bftheta}-\lambda\br{\bftheta}. \label{eq:cert_general}
% \end{align}
\begin{align}
    \sup_{\substack{\bftheta \in \Thetaspace \\ \bfzadv \in \Aspace}} \quad \ExP{\bftheta^\prime \sim \Pi(\cdot | \bftheta, \bfzadv)}{\lambda\br{\bftheta^\prime}} + \ell_{\textup{adv}}\br{\bftheta}-\lambda\br{\bftheta}. \label{eq:cert_general}
\end{align}
% \end{theorem}
\end{restatable}
\begin{proof}
% \avi{Done for space constraints}
% \textit{Proof.}
Follows by weak duality for the LP
\eqref{eq:lp_inf}. Detailed proof in Appendix~\ref{appendix:proofs}.
% \maryam{let's refer to appendix and write it there carefully with citation in inf-dimensional duality}
\end{proof}
%\eqref{eq:certificate}.
% \maryam{Give more details on derivation (or point to appendix)}

If strong duality holds \cite{nash1987linear, clark2003infinite},
% \maryam{Cite references that properly cover infinite dimensional duality and constraint qualifications. Maybe the book by Attouche et al on optimization in Sobolev spaces? Other refs?}
we further have that the optimal value of \eqref{eq:lp_inf} is exactly equal to
% \begin{subequations}
\begin{align}
\hspace{-.18in}
{\tiny{\inf_{\lambda: \Thetaspace \mapsto \R} \sup_{\substack{\bftheta \in \Thetaspace \\ \bfzadv \in \Aspace}} \ExP{\bftheta^\prime \sim \Pi(\cdot | \bftheta, \bfzadv)}{\lambda\br{\bftheta^\prime}} + \ell_{\textup{adv}}\br{\bftheta}-\lambda\br{\bftheta}. }}\label{eq:certificate}
\end{align}

\section{MEAN ESTIMATION}\label{sec:mean_estimation}
% We consider the special case of mean estimation where $\ell_{\textup{adv}}, \ell\br{\theta}$ are quadratic. In particular,
Consider the mean estimation problem, where we aim to learn the parameter $\bftheta \in \R^d$ to estimate the mean $\mu = \ExP{\bfz \sim \Pdata}{\bfz}$ of a distribution $\Pdata$. Given a data point $\bfz_t$, the learning rule is given by:
\begin{align}\label{eq:mean_update}
    \bftheta_{t+1} \gets (1 - \eta)\bftheta_t + \eta \bfz_t +  \eta \bfB \bfw_t,
\end{align}
where $\eta$ is the learning rate, $\bfS = \bfB \bfB^\top \in \mathbb{S}^d_{+}$ is the tunable defense hyperparameter and $\bfw_t \sim \mathcal{N}(0, \bfI)$ is Gaussian noise. The adversary wants to maximize its average reward according to the following objective function:
\begin{align}
\ell_{\textup{adv}}\br{\bftheta} = \norm{\mu-\bftheta}^2.
\end{align}
% Our goal is minimizing the following objective:
% \avi{Overloaded symbol $\mu$}
% \[\ell\br{\bftheta} = \ExP{\bfz \sim \Pdata}{\norm{\bfz-\bftheta}^2} = \norm{\bftheta -\mu}^2 + \ExP{\bfz \sim \Pdata}{\bfz\bfz^\top} - \mu^2\]

% Upon observing a sample $\bfz$, the learning algorithm updates the parameter $\theta$ as follows:
% \begin{align*}
%     \theta \leftarrow \theta - \eta(\bftheta - \bfz) + \eta \bfB \bfw,
% \end{align*}
% where $\bfw \sim \mathcal{N}(0, \bfI)$, $\eta$ denotes the learning rate and $\bfB\bfB^\top = \bfS$.
% 2\tran{\mu}\bftheta + \ldots \]

% To avoid detection the adversary poisons samples from within a constraint set. Suppose the constraint set $\mathcal{A}$ is defined by ${(\bfzadv - \mu)}^\top (\bfzadv - \mu) \leq c$.
% \avi{We were using $F(\bftheta, \bfz)$ for the update rule form earlier, lets make it consistent.}
% \djcomment{Sounds good, feel free to make this change}
% \subsection{Certificate on adversarial loss (analysis)}
\paragraph{Certificate on adversarial loss (analysis).}

\begin{restatable}{theorem}{meancertificate}
Choosing $\lambda : \R^d \rightarrow \R$ in Theorem~\ref{thm:certificate} to be quadratic, i.e.  $\lambda\br{\bftheta} = \bftheta^\top \bfA \bftheta + \bftheta^\top \bfb$, the adversarial constraint set of the form $\|\bfzadv - \mu\|_2^2 \leq r$, %\maryam{did we decide to change this?}\avi{No. It was in classification setting.},
the certificate for the mean estimation problem for $\Pdata(\bfz) = \mathcal{N}(\bfz | \mu, \bfSigma)$
%\maryam{$\mathcal{N}(\mu,\Sigma)$ is more standard if no confusion about the random variable}
for a fixed learning algorithm (i.e. $\bfS$ is fixed) is given by:
\begin{align}
    \inf_{\bfA \in \mathbb{S}^d, \bfb \in \R^d, \nu \geq 0} g(\bfA, \bfb,  \nu, \bfS, \mu, \bfSigma),\label{eq:mean_certificate}
\end{align}
where $g(\bfA, \bfb,  \nu, \bfS, \mu, \bfSigma)$ is a convex objective in $\bfA,\bfb,\nu$  as defined below:
% \begin{align}
%     g(\bfA, \bfb,  \nu, \bfS, \mu, \bfSigma) &=
%     \begin{cases}
%     &  \frac{1}{4}\bmat{\bfd_3 \\ \epsilon \eta \bfb - 2 \nu \mu}^\top\br{- \bmat{ \bfD_1 & \bfD_2 / 2 \\ \bfD_2 / 2 &  \epsilon\eta^2 \bfA + \nu \bfI}}^{-1}\bmat{\bfd_3 \\ \epsilon \eta \bfb - 2 \nu \mu}  +  d_4 + \nu(\mu^\top \mu - r) \\& \text{if} \quad \nu \leq 0 ;\; \bmat{ \bfD_1 & \bfD_2 / 2 \\ \bfD_2 / 2 &  \epsilon\eta^2 \bfA + \nu \bfI} \preceq 0 \\
%     & -\infty \quad \text{else}
%     \end{cases}\label{eq:dualfunction}
% \end{align}
\begin{equation}
    % g(\bfA, \bfb,  \nu, \bfS, \mu, \bfSigma)
    \begin{cases}
    \frac{1}{4} \left\|\bmat{2(1 - \epsilon)\eta(1 - \eta)\bfA \mu - 2 \mu - \eta \bfb \\ \epsilon \eta \bfb + 2 \nu \mu}\right\|^2_{\bfD} \\
    \;\;+ (1-\epsilon)(\eta^2 \Trace(\bfSigma \bfA) + \eta^2 \mu^\top \bfA \mu + \eta \bfb^\top \mu)\hspace*{-3mm}\\
    \qquad + \mu^\top \mu + \eta^2 \Trace(\bfA\bfS) + \nu(r - \mu^\top \mu),\\
    \qquad\qquad \text{if }\nu \geq 0 \text{ and } \bfD \succeq 0 \\
    -\infty\qquad\; \text{otherwise,}
    \end{cases}\label{eq:dualfunction}
\end{equation}
where $\bfD = \bmat{ (1 - (1 - \eta)^2)\bfA - \bfI & -\eta\epsilon(1 - \eta)\bfA \\ -\eta\epsilon(1 - \eta)\bfA &  -\epsilon\eta^2 \bfA - \nu \bfI}$ and $\|\bfx\|^2_{\bfD} = \bfx^\top \bfD^{-1} \bfx$.
% and $\bfD_1 = ((1 - \eta)^2 - 1)\bfA + \bfI, \bfD_2 = 2\eta\epsilon(1 - \eta)\bfA, \bfd_3 = 2(1 - \epsilon)\eta(1 - \eta)\bfA \mu - 2 \mu - \eta \bfb, d_4 = (1-\epsilon)(\eta^2 \Trace(\bfSigma \bfA) + \eta^2 \mu^\top \bfA \mu + \eta \bfb^\top \mu) + \mu^\top \mu + \eta^2 \Trace(\bfA\bfS)$.
\end{restatable}
\begin{proof}
    The detailed proof can be found in Appendix~\ref{appendix:mean_estimation}. The proof sketch follows by noting that for a fixed $\lambda$ the dual of the inner constrained maximization problem is a quadratic expression in $\bfzadv, \bftheta$ and a finite supremum exists if the Hessian is negative semidefinite (which leads to the $\mathbf{D} \succeq 0$ constraint). Plugging in the maximizer we get a minimization problem in the dual variables $\nu$. Finally we note that the overall minimization problem is jointly convex in $\bfA, \bfb, \nu$.
\end{proof}
\begin{remark}
The problem above
%~\eqref{eq:dualfunction}
is a convex problem since it has a matrix fractional objective function \cite{boyd2004convex}
% \maryam{(add citation, can cite chapter of boyd's book)}\avi{Sounds good. Do we want citation in the theorem statement?} \maryam{after the theorem.}
with a Linear Matrix Inequality (LMI) constraint.
\end{remark}
% \end{theorem}
% \textbf{Note:} The closed form of the expression in the if case will have a pseudo-inverse of $\nu$. Can we still feed the expression below to cvxpy?
% \subsection{Defense for a single data distribution}
% \paragraph{Defense for a single data distribution}
% % \textbf{Parameters for robust learning algorithm:}
% An upper bound on the objective in Eq. \eqref{eq:mean_certificate} for a single given $\Pdata$ can thus be written as the following minimization problem:
% \begin{align}
%     \inf_{\substack{\bfS \in \mathbb{S}^d_+ \\ \bfA \in \mathbb{S}^d, \bfb \in \R^d \\ \nu \leq 0}}\qquad &\eta^2 \Trace(\bfS)  + \kappa  g(\bfA, \bfb,  \nu, \bfS, \mu, \bfSigma) \label{eq:mean_defense}
% \end{align}
% \begin{align}
%     \inf_{\substack{\bfS \in \mathbb{S}^d_+ \\ \bfA \in \mathbb{S}^d, \bfb \in \R^d \\ \nu \leq 0}}\qquad &\eta^2 \Trace(\bfS)  + \kappa (\frac{1}{4} \bmat{\bfd_3 \\ \epsilon \eta \bfb}^\top \br{- \bmat{ \bfD_1 & \bfD_2 / 2 \\ \bfD_2 / 2 &  \epsilon\eta^2 \bfA + \nu \bfI}}^{-1}\bmat{\bfd_3 \\ \epsilon \eta \bfb}  +  d_4 - \nu c) \label{eq:mean_estimation} \\
%     \text{s.t.} \quad  &\bmat{ \bfD_1 & \bfD_2 / 2 \\ \bfD_2 / 2 &  \epsilon\eta^2 \bfA + \nu \bfI} \preceq 0.
% \end{align}
% \subsection{Meta-Learning Algorithm}

\paragraph{Meta-Learning Algorithm.}
% \textbf{Parameters for meta robust learning algorithm (expected loss):}
Following the formulation in Eq.~\eqref{eq:main_formulation}, we wish to learn a defense parameter $\bfS$ that minimizes the expected loss (expectation over different $\Pdata$ from the meta distribution $\mathcal{P}$). For the mean estimation problem this boils down to solving
\begin{align}
    \inf_{\substack{\bfS \in \mathbb{S}^d_+ }} &\eta^2 \Trace(\bfS)  + \kappa \mathbb{E}_{\mu, \bfSigma \sim \mathcal{P}} [ \!\!\!\! \inf_{\substack{\nu \geq 0 \\ \bfA \in \mathbb{S}^d, \bfb \in \R^d}} \!\!\!\! g(\bfA, \bfb,  \nu, \bfS, \mu, \bfSigma)].\label{eq:mean_meta_defense}
\end{align}
\vspace{-.05in}
\begin{remark}
    Note that problem \eqref{eq:mean_meta_defense} is not jointly convex in $\bfS,\bfA,\bfb,\nu$ because of the $\Trace(\bfA\bfS)$ term in $g$; see \eqref{eq:dualfunction}. However, it is convex individually in $\bfS$ and $\{\bfA,\bfb,\nu\}$. We use an alternating minimization approach to seek a local minimum of problem~\eqref{eq:mean_meta_defense}, as detailed in Algorithm~\ref{alg:meta_mean}.
\end{remark}
% \LL{maybe mention that \eqref{eq:mean_meta_defense} is not quite jointly convex in $\bfS,\bfA,\bfb,\nu$ because of the $\Trace(\bfA\bfS)$ term?}\avi{Note that Eq.~\eqref{eq:mean_meta_defense} is not jointly convex in $\bfA, \bfS$ because of the $\Trace(\bfA\bfS)$ term. We use an alternating minimization procedure to %solve the
% find the certificate for a fixed $\bfS$, and then minimize Eq~\eqref{eq:mean_meta_defense} over $\bfS$ keeping $\bfA, \bfb, \nu$ fixed at the solutions of the certificate.} \maryam{let's state clearly that this is not optimal, since the problem is not convex and doesn't have any other structure we can use to prove anything about alternating minimization.}\avi{Not sure how adding this improves the paper? We can simply say that the trace causes non-convexity and we adapt alternating minimization to get to a local optima.}

In practice, one observes a finite number of distributions from $\mathcal{P}$, and sample average approximation is leveraged, with the aim of learning a defense parameter which generalizes well to unseen distributions from $\mathcal{P}$. This process is stated in Algorithm~\ref{alg:meta_mean}.

\begin{algorithm}[t]
    \caption{Meta learning a robust learning algorithm for mean estimation}
    \begin{algorithmic}[1]
        \STATE \textbf{Input: } Set of $K$ distributions $\{\mathcal{N}(\mu_i, \bfSigma_i)\}_{i \in [K]}$ sampled from $\mathcal{P}[\Pdata]$, tradeoff parameter $\kappa$, and max iterations $T$.
        \STATE \textbf{Initialize:} $\bfS^{(1)} \in \mathbb{S}^d_+$ randomly.
        % \maryam{name this $S_0$ or $S_1$, add proper indices to $S$}
        \STATE \textit{Alternating Minimization over Lagrange %parameters
        multipliers $\{\bfA_i, \bfb_i, \nu_i\}_{i \in [K]}$ and defense parameter $\bfS$.}
        \FOR{$t=1, \ldots, T$}
        \FOR{$i=1,\ldots, K$}
        \STATE $\bfA_i, \bfb_i, \nu_i = \underset{\bfA \in \mathbb{S}^d, \bfb \in \R^d, \nu \geq 0}{\argmin} g(\bfA, \bfb, \nu, \bfS^{(t)}, \mu_i, \bfSigma_i)$
         \ENDFOR
        \STATE $\bfS^{(t+1)} = \underset{\bfS \in \mathbb{S}^d_+}{\argmin}\bigl(\frac{\kappa}{K} \sum_{i \in [K]} g(\bfA_i, \bfb_i, \nu_i, \bfS, \mu_i, \bfSigma_i)$ \\[-3mm]
        \hspace{5cm}$+  \eta^2 \Trace(\bfS)\bigr)$
        % \IF{$\bfS^{(t+1)} = \bfS^{(t)}$}
        % \STATE break
        % \ENDIF
        \ENDFOR
        \RETURN $\bfS^{(T+1)}$
    \end{algorithmic}
    \label{alg:meta_mean}
\end{algorithm}
% \vspace{-5mm}

\section{BINARY CLASSIFICATION}\label{sec:binary_classification}
% We consider the binary classification problem. Given an input feature $\bfx \in [-1, 1]^d$, a linear predictor $\bftheta \in  \R^d$ tries to predict the label $y \in \{-1, 1\}$ via the sign of $\bftheta^\top \bfx$. To define the losses, we introduce $\bfz = y \bfx \in [-1,1]^d$. We consider the regularized hinge loss:
We consider the binary classification problem. Given an input feature $\bfx \in \R^d$ such that $\|\bfx\|_2 \leq 1$, a linear predictor $\bftheta \in  \R^d$ tries to predict the label $y \in \{-1, 1\}$ via the sign of $\bftheta^\top \bfx$ . To define the losses, we introduce $\bfz = y \bfx \in \R^d$ which is the label multiplied by the feature and note that $\|\bfz\|_2 \leq 1$.

A dynamic adversary tries to corrupt samples so that the learning algorithm learns a $\bftheta$ that maximizes the hinge loss on a target distribution $\Ptarget$, captured by the following objective:
\begin{align}\label{eq:adv_classification}
    \ell_{\textrm{adv}}(\bftheta) = \ExP{z \sim \Ptarget}{\max \{0, 1 - \bftheta^\top \bfz\}}.
\end{align}
% \avi{Very easy to also do it for the case of 0-1 loss, $\ell_{\textrm{adv}}(\bftheta) = \ExP{z \sim \Pdata}{\mathbb{I}[- \bftheta^\top \bfz \geq 0]}$. This is just a measure of classification accuracy.}
% \begin{align}
%     h(x) = \mathbb{I}[x \leq 1]
% \end{align}
% \textbf{Logistic Loss:}
% \begin{align}
%     l(\bftheta, \bfz) = -\log(1 + \exp{(-\bftheta^\top \bfz)})
% \end{align}
% \begin{align}
%     h(x) =  -\frac{1}{\exp(x) + 1}
% \end{align}
% \begin{align}
%     F(\bftheta, \bfz) = \bftheta - \eta \frac{\bfz}{\exp{(\bftheta^\top \bfz)} + 1}.
% \end{align}
%
The learning algorithm tires to minimize the regularized hinge loss on the observed datapoints:
\begin{align}\label{eq:hinge}
    l(\bftheta, \bfz) &= \max \{0, 1 - \bftheta^\top \bfz\} + \frac{\sigma}{2}\norm{\bftheta}_2^2.
\end{align}
% Upon observing a sample $\bfz_t$, the parameter is updated via a gradient descent $\bftheta_{t+1} = \bftheta_{t} - \eta \nabla_{\bftheta} l(\bftheta_t, \bfz_t) = F(\bftheta_t, \bfz_t)$ where $F(\bftheta, \bfz)$ is defined as:
% \begin{align}
%     F(\bftheta, \bfz) = (1 - \sigma \eta)\bftheta + \eta \mathbb{I}[\bftheta^\top \bfz \leq 1] \bfz.
% \end{align}
Upon observing a sample $\bfz_t$, the parameter is updated via a gradient descent: $\bftheta_{t+1} = F(\bftheta_t, \bfz_t)$, where \begin{align}
    F(\bftheta, \bfz) &= \bftheta_{t} - \eta \nabla_{\bftheta} l(\bftheta_t, \bfz_t) \notag \\
    &= (1 - \sigma \eta)\bftheta + \eta \mathbb{I}[\bftheta^\top \bfz \leq 1] \bfz. \label{eq:hinge_online_learning}
\end{align}
Below, we provide a certificate for the adversarial objective at stationarity of this learning algorithm.
% \avi{How to specify $\mathcal{A}$? I know we discussed a likelihood model with assumption that features (here label multiplied by feature) will yield a $\ell_2$ norm ball. However, typically constraint sets are perturbations of individual data points. Is there any way of making it more consitent?}
% \djcomment{I think norm balls are a fine starting point modeling the assumption that features are normalized to within a unit ball before the algorithm is run, but we make no additional assumptions. It would be good though to look through the literature and see if anything else has been done, in particularly in \citep{sosnin2024certified}}
\begin{restatable}{theorem}{classificationhingecertificate}
\label{thm:classificationhingecertificate}
Choosing $\lambda : \R^d \rightarrow \R$ in Theorem~\ref{thm:certificate} to be quadratic, i.e., $\lambda\br{\bftheta} = \bftheta^\top \bfA\bftheta + \bfb^\top \bftheta$, parameter space $\Thetaspace = \{ \bftheta \mid \|\bftheta\|_2 \leq \frac{1}{\sigma}\}$ and the adversarial constraint set of the form $\mathcal{A} = \{ \bfzadv \mid \|\bfzadv\|_{2    }
\leq 1\}$, the certificate for the binary classification problem for $\Pdata(\bfz) = \{\bfz_1, \ldots, \bfz_N\}$ for a learning algorithm with regularization parameter $\sigma$, and learning rate $\eta$ is upper bounded by
\begin{align*}
    &\max \left(\OPT_1, \OPT_2 \right)\\[3mm]
\OPT_1 = \inf_{\nu, \bfA, \bfb} \quad& \|\bfp(\bfb, \nu)\|^2_{\bfD(\bfA, \nu)^{-1}} + q(\nu)\\
\textup{s.t.}  \quad& \bfD(\bfA, \nu) \succeq 0,\\
&\bfr_i(\nu) + \bfs(z_i, \bfA, \bfb) = 0, \;\; \forall i \in [N].\\[3mm]
\OPT_2 = \inf_{\nu, \bfA, \bfb} \quad& \|\bfp'(\bfb, \nu)\|^2_{\bfD'(\bfA, \nu)^{-1}} + q'(\nu)\\
\textup{s.t.}  \quad&  \bfD'(\bfA, \nu) \succeq 0,\\
      &\bfr_i(\nu) + \bfs(z_i, \bfA, \bfb) = 0, \;\; \forall i \in [N].
\end{align*}
% \begin{align*}
%     &\max{\left(\inf_{\nu, \bfA, \bfb} \|\bfp(\bfb, \nu)\|^2_{\bfD(\bfA, \nu)^{-1}} + q(\nu), \inf_{\nu, \bfA, \bfb} \|\bfp'(\bfb, \nu)\|^2_{\bfD'(\bfA)^{-1}} + q'(\nu)\right)} \\
%     &\text{such that} \\
%     & \bfr_i(\nu) + \bfs(z_i, \bfA, \bfb) = 0 \; \forall i \in [N],\\
%     & \bfD(\bfA, \nu) \succeq 0, \bfD'(\bfA) \succeq 0.
% \end{align*}

where $\bfp()$, $\bfp'()$, $q()$, $q'()$, $\bfD()$, $\bfD'()$, $\bfr_1(),\ldots,\bfr_N()$, $\bfs()$ are affine functions of the optimization variables $\nu, \bfA, \bfb$ as defined below and $\nu = \{\nu_1, \ldots, \nu_{10}\}$:
\footnotesize
% \begin{align*}
%     \bfp(\bfb, \nu) &= \frac{1}{2}\bmat{-\sigma \eta \bfb + \sum_{i \in [N]} \left((\nu_{1i} - \nu_{2i}) \bfz_i - \nu_{4i} +\nu_{6i} \right) \\ \epsilon \eta \bfb} \in \R^{2d},\\
%     \bfD(\bfA, \nu) &= \addtolength{\arraycolsep}{-1mm}\bmat{[1 - (1 - \sigma \eta)^2] \bfA + \nu_8 \bfI & -\epsilon (1 - \sigma \eta) \eta \bfA + \nu_9 \bfI\\ -\epsilon (1 - \sigma \eta) \eta \bfA + \nu_9 \bfI & -\epsilon \eta^2 \bfA + \nu_{10}\bfI} \in \mathbb{S}_{+}^{2d},\\
%     % {2d \times 2d}\\
%     q(\nu) &= q'(\nu) + 2 \nu_9 + \nu_{10} \in \R, \\
%     \bfp'(\bfb, \nu) &= \frac{1}{2}\bmat{-\sigma \eta \bfb + \sum_{i \in [N]} \left((\nu_{1i} - \nu_{2i}) \bfz_i - \nu_{4i} +\nu_{6i} \right)} \in \R^{d},\\
%     \bfD'(\bfA, \nu) &= [1 - (1 - \sigma \eta)^2] \bfA + \nu_8 \bfI  \in \mathbb{S}_{+}^{d},\\
%     % {2d \times 2d}\\
%     q'(\nu) &=-\nu_1^\top \mathbf{1} + (2    + \tfrac{1}{\sigma}) \nu_2^\top \mathbf{1} + \tfrac{(\nu_4 + \nu_6)^\top \bf1}{\sigma} + \mathbf{1}^\top \nu_7  + \frac{\nu_{8}}{\sigma^2} \in \R, \\
%     \bfr_i(\nu) &= \bmat{(1 + \frac{1}{\sigma}) (\nu_{1i} - \nu_{2i})  - \frac{\mathbf{1}^\top(\nu_{4i} + \nu_{6i})}{\sigma} - \nu_{7i}\\ \nu_{3i} + \nu_{4i} - \nu_{5i} - \nu_{6i}} \in \R^{d + 1},\\
%     \bfs(\bfz_i, \bfA, \bfb) &= \frac{1}{N} \bmat{(1 - \epsilon) \eta^2 \bfz_i^\top \bfA \bfz_i + (1 - \epsilon) \eta \bfb^\top \bfz_i + 1\\ 2(1 - \epsilon) \eta (1 - \sigma \eta) \bfA \bfz_i - \bfz_i} \in \R^{d + 1}.
% \end{align*}
\begin{align*}
    \bfp(\bfb, \nu) &= \frac{1}{2}\bmat{-\sigma \eta \bfb + \sum_{i \in [N]} \left((\nu_{1i} - \nu_{2i}) \bfz_i - \nu_{4i} +\nu_{6i} \right) \\ \epsilon \eta \bfb},\\
    \bfD(\bfA, \nu) &= \addtolength{\arraycolsep}{-1mm}\bmat{[1 - (1 - \sigma \eta)^2] \bfA + \nu_8 \bfI & -\epsilon (1 - \sigma \eta) \eta \bfA + \nu_9 \bfI\\ -\epsilon (1 - \sigma \eta) \eta \bfA + \nu_9 \bfI & -\epsilon \eta^2 \bfA + \nu_{10}\bfI},\\
    % {2d \times 2d}\\
    q(\nu) &= q'(\nu) + 2 \nu_9 + \nu_{10}, \\
    \bfp'(\bfb, \nu) &= \frac{1}{2}\bmat{-\sigma \eta \bfb + \sum_{i \in [N]} \left((\nu_{1i} - \nu_{2i}) \bfz_i - \nu_{4i} +\nu_{6i} \right)},\\
    \bfD'(\bfA, \nu) &= [1 - (1 - \sigma \eta)^2] \bfA + \nu_8 \bfI,\\
    % {2d \times 2d}\\
    q'(\nu) &=-\nu_1^\top \mathbf{1} + (2    + \tfrac{1}{\sigma}) \nu_2^\top \mathbf{1} + \tfrac{(\nu_4 + \nu_6)^\top \bf1}{\sigma} + \mathbf{1}^\top \nu_7  + \frac{\nu_{8}}{\sigma^2}, \\
    \bfr_i(\nu) &= \bmat{(1 + \frac{1}{\sigma}) (\nu_{1i} - \nu_{2i})  - \frac{\mathbf{1}^\top(\nu_{4i} + \nu_{6i})}{\sigma} - \nu_{7i}\\ \nu_{3i} + \nu_{4i} - \nu_{5i} - \nu_{6i}},\\
    \bfs(\bfz_i, \bfA, \bfb) &= \frac{1}{N} \bmat{(1 - \epsilon) \eta^2 \bfz_i^\top \bfA \bfz_i + (1 - \epsilon) \eta \bfb^\top \bfz_i + 1\\ 2(1 - \epsilon) \eta (1 - \sigma \eta) \bfA \bfz_i - \bfz_i}.
\end{align*}
\end{restatable}
\begin{proof}
A detailed derivation can be found in the Appendix~\ref{appendix:binary_classification}. The proof steps are: (i) Regularization implicitly bounds the decision variable $\bftheta \in \Thetaspace$ without the need for projections, (ii) Considering 2 cases for the indicator term of $\bfzadv$ and taking the maximum of these 2 cases, (iii) Relaxing the indicator terms for benign samples into continuous variables between $[0, 1]$, and (iv) Using McCormick relaxations \citep{mitsos2009mccormick} to convexify the bilinear terms in the objective.
\end{proof}
% \paragraph{Robust Learning.} Unlike the problem of mean estimation, the certificate is not convex in the tunable hyper-parameters $\nu, \sigma$. Hence approaches gradient descent or grid search can be used to effectively tune the defense parameters. We carry extensive experiments in Section~\ref{sec:experiments} to verify the efficacy on real datasets.
\textbf{Extension to Multi-Class Setting:} Our analysis can be extended to the multi-class classification setting. Let us consider score based classifiers, where $\Theta = \{\theta_1, \ldots, \theta_K\} \in \mathbb{R}^d$ are the learnable parameters and the class prediction for a feature $x \in \mathbb{R}^d$ is given by $\arg\max_{i \in [K]}\theta_i^\top x$.

The SVM loss for any arbitrary feature $x$ with label $y$ is defined as:
$\ell(\Theta, (x, y)) = \sum_{j \neq y} \max\{0, 1 + (\theta_j - \theta_y)^\top x\}$

The gradient update takes the form:
$F(\theta_y, (x, y)) = \theta_y + \eta \sum_{j \neq y} \mathbb{I}[1 + (\theta_j - \theta_y)^\top x > 0] x$ and for all $j \neq y$, we have $F(\theta_j, (x, y)) = \theta_j - \eta \mathbb{I}[1 + (\theta_j - \theta_y)^\top x > 0] x$.

Note that the non-linearity in both the loss function and the update is composed with a linear combination of the parameters (i.e. $\theta_j - \theta_y$), and thus the analysis in the proof of Theorem~\ref{thm:classificationhingecertificate} still holds, and our analysis for the binary classification generalizes to the multi-class classification.

\section{EXPERIMENTS}\label{sec:experiments}
We conduct experiments on both synthetic and real data to empirically validate our theoretical tools.
\begin{figure}[h]
    \centering
    \begin{minipage}{0.32\textwidth}
        \centering
        \includegraphics[width=\textwidth]{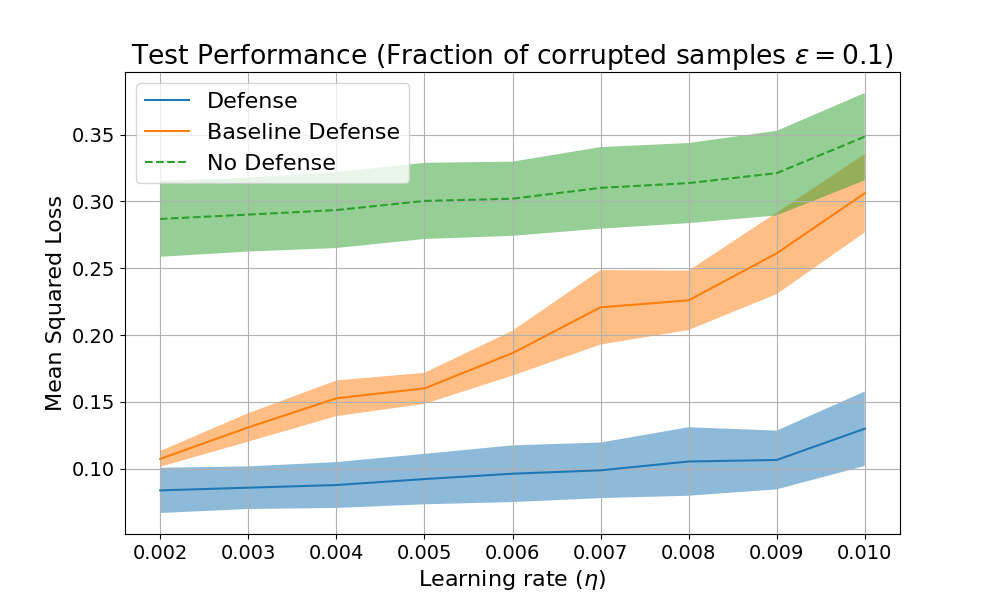} % First image file
        % \caption{$d=2$}
        \label{fig:image1}
    \end{minipage}
    \vskip -0.2in
    \begin{minipage}{0.32\textwidth}
        \centering
        \includegraphics[width=\textwidth]{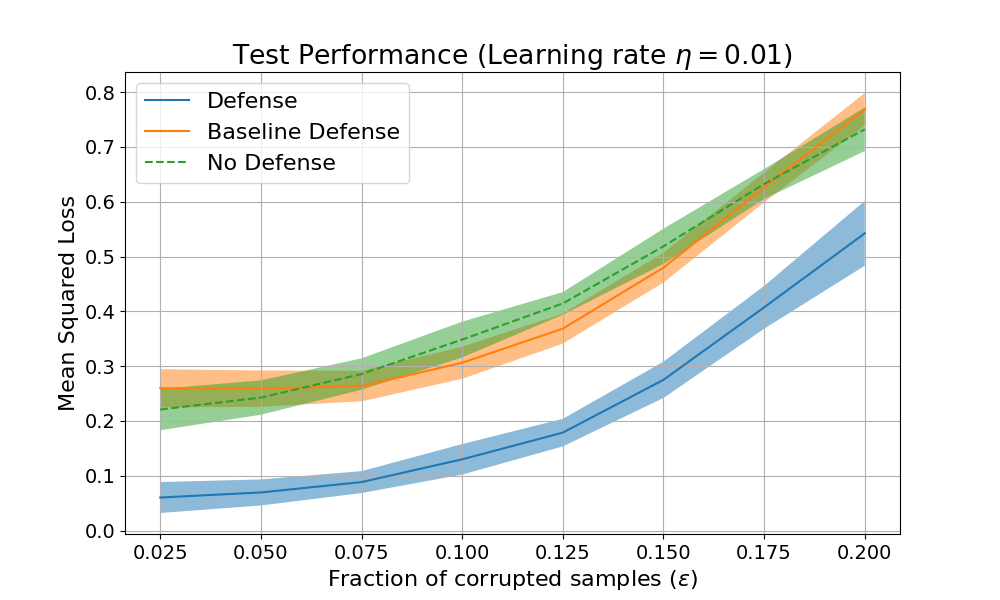} % Second image file
        % \caption{$d=3$}
        \label{$d=3$}
    \end{minipage}
    \vskip -0.2in
    \caption{Test performance (mean squared error between true and estimated means) upon varying the learning rates (above) and the the fraction of samples corrupted by the dynamic adversary (below) and observed that our defense significantly outperforms training without defense.}
    \vskip -0.2in
    \label{fig:mean_estimation}
\end{figure}
% \begin{figure*}[h!]
%     \centering
%     \begin{minipage}{0.32\textwidth}
%         \centering
%         \includegraphics[width=1.1\textwidth]{figs/attack_cifar.png} % First image file
%         % \caption{$d=2$}
%         \label{fig:image1}
%     \end{minipage}
%     % \vskip -0.15in
%     \begin{minipage}{0.32\textwidth}
%         \centering
%         \includegraphics[width=1.1\textwidth]{figs/attack_fashion.png} % Second image file
%         % \caption{$d=3$}
%         \label{$d=3$}
%     \end{minipage}
%     % \vskip -0.15in
%     \begin{minipage}{0.32\textwidth}
%         \centering
%         \includegraphics[width=1.1\textwidth]{figs/attack_mnist_1_vs_7.png} % Second image file
%         % \caption{$d=3$}
%         \label{$d=3$}
%     \end{minipage}
%     \vskip -0.15in
%     \caption{We plot the certificates of robustness for various settings (hyperparameter values) which act as upper bounds on the optimal dynamic adversary's objective. We also plot the test losses on the adversarial objective  for various attacks which act as lower bounds on the objective of the optimal adversary. }
%     \vskip -0.2in
%     \label{fig:classification}
% \end{figure*}
\begin{figure*}[h!]
    \centering
    \begin{minipage}{0.32\textwidth}
        \centering
        \includegraphics[width=1.1\textwidth]{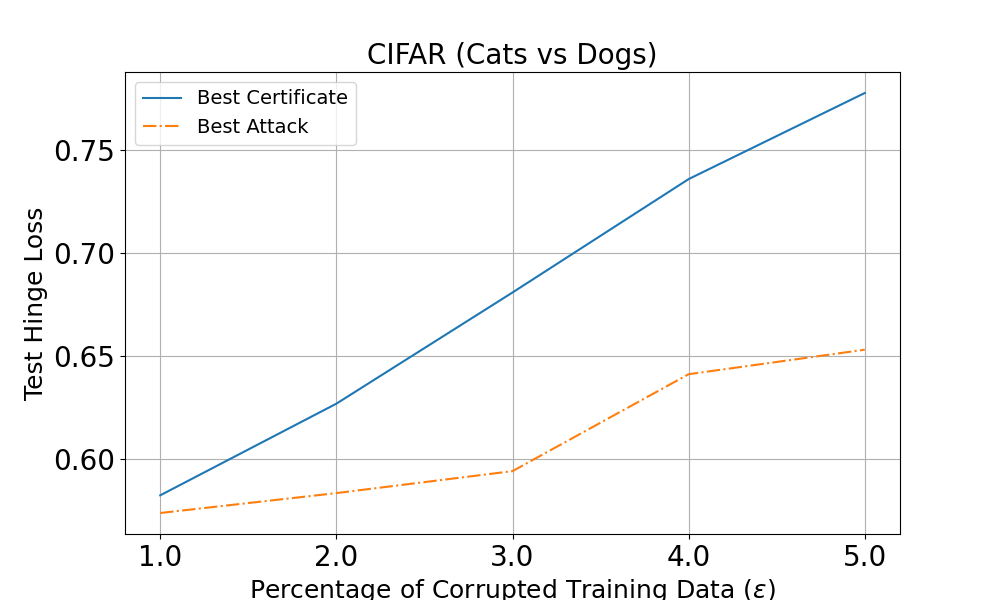} % First image file
        % \caption{$d=2$}
        \label{fig:image1}
    \end{minipage}
    % \vskip -0.15in
    \begin{minipage}{0.32\textwidth}
        \centering
        \includegraphics[width=1.1\textwidth]{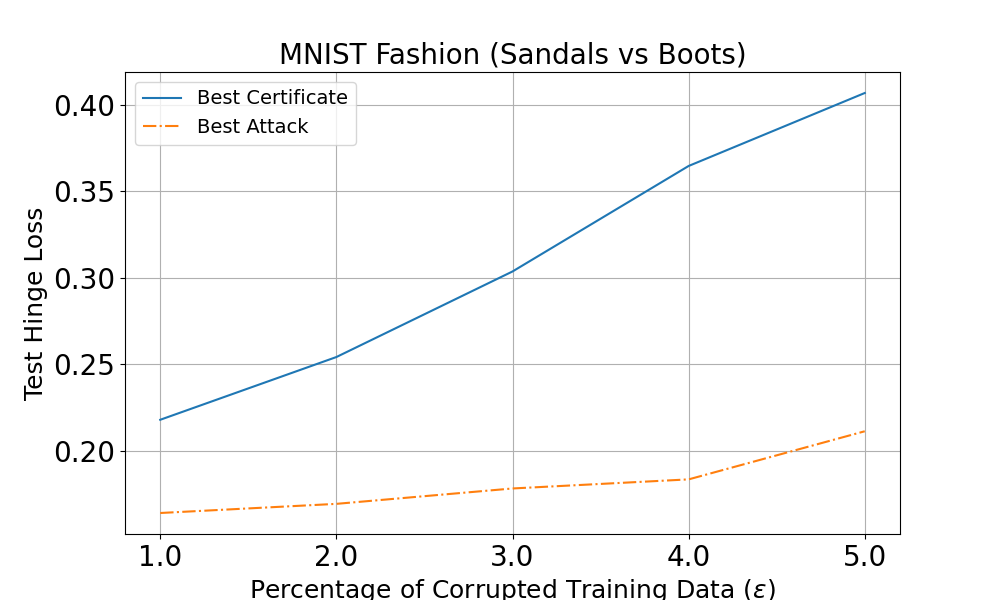} % Second image file
        % \caption{$d=3$}
        \label{$d=3$}
    \end{minipage}
    % \vskip -0.15in
    \begin{minipage}{0.32\textwidth}
        \centering
        \includegraphics[width=1.1\textwidth]{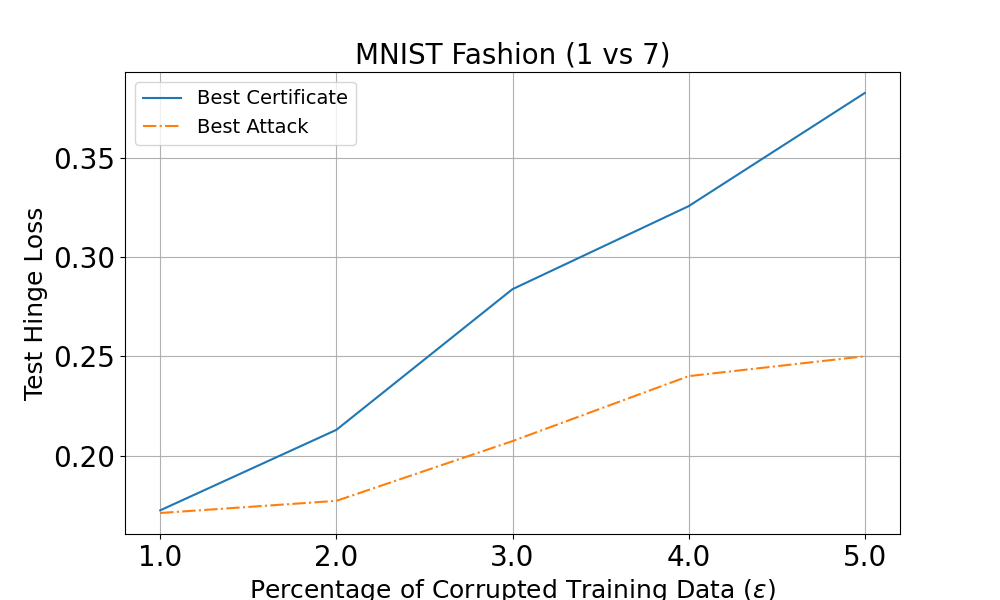} % Second image file
        % \caption{$d=3$}
        \label{$d=3$}
    \end{minipage}
    \vskip -0.15in
    \caption{We plot the certificates of robustness for various settings (hyperparameter values) which act as upper bounds on the optimal dynamic adversary's objective. We also plot the test losses on the adversarial objective  for various attacks which act as lower bounds on the objective of the optimal adversary. }
    \vskip -0.2in
    \label{fig:classification}
\end{figure*}
% \begin{figure}[h!]
%     \centering
%     \begin{minipage}{0.32\textwidth}
%         \centering
%         \includegraphics[width=\textwidth]{figs/attack_help.png} % First image file
%         % \caption{$d=2$}
%         \label{fig:image1}
%     \end{minipage}
%     \vskip -0.15in
%     \begin{minipage}{0.32\textwidth}
%         \centering
%         \includegraphics[width= \textwidth]{figs/attack_correct.png} % Second image file
%         % \caption{$d=3$}
%         \label{$d=3$}
%     \end{minipage}
%     \vskip -0.15in
%     \caption{Poor choice of hyperparameters of the learning algorithm can make them vulnerable to dynamic attackers as noted by our certificates and attacks (red plots). Lower values of certificate, indicate more robust learning algorithms (blue, orange, green plots).}\label{fig:llm}
%     \vskip -0.35in
% \end{figure}
\begin{figure}[h!]
    \centering
    \begin{minipage}{0.32\textwidth}
        \centering
        \includegraphics[width=\textwidth]{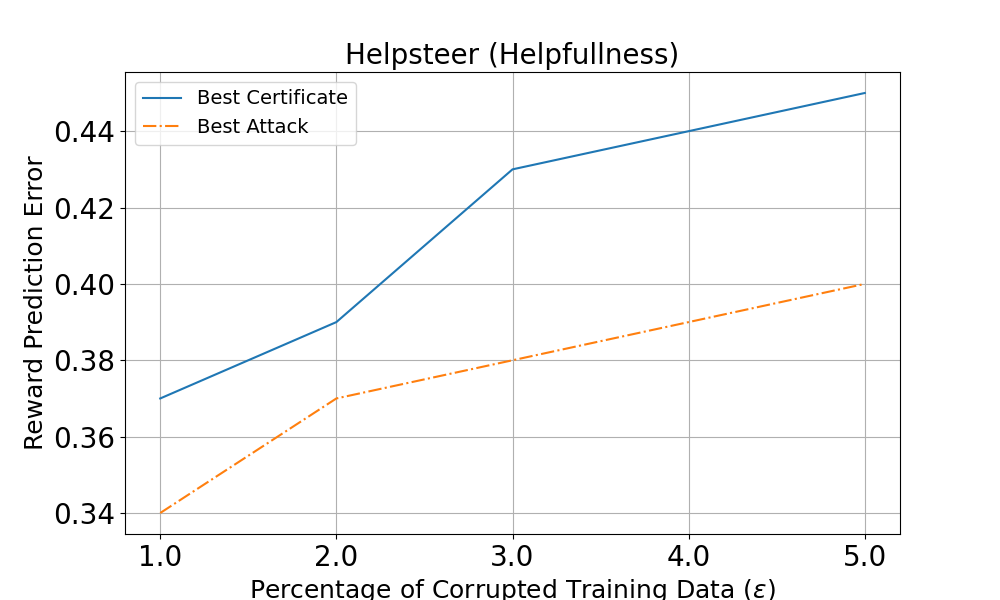} % First image file
        % \caption{$d=2$}
        \label{fig:image1}
    \end{minipage}
    \vskip -0.15in
    \begin{minipage}{0.32\textwidth}
        \centering
        \includegraphics[width= \textwidth]{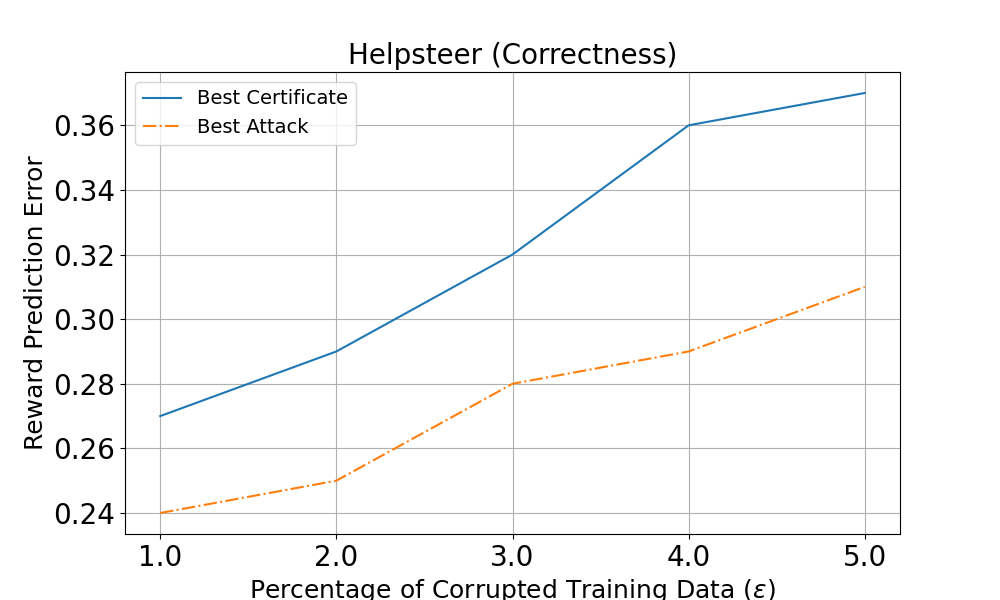} % Second image file
        % \caption{$d=3$}
        \label{$d=3$}
    \end{minipage}
    \vskip -0.15in
    \caption{Poor choice of hyperparameters of the learning algorithm can make them vulnerable to dynamic attackers as noted by our certificates and attacks (red plots). Lower values of certificate, indicate more robust learning algorithms (blue, orange, green plots).}\label{fig:llm}
    \vskip -0.2in
\end{figure}
\subsection{Mean Estimation}
\vskip -0.15in
We wish to evaluate the robustness of our meta learning algorithm in Eq.~\eqref{eq:main_formulation} to design a defense against a dynamic best responding adversary on a ($d=20$) mean estimation task. We consider 3 different learning algorithms: 1. \textbf{No Defense:} Eq.~\eqref{eq:mean_update} with $\bfB=\mathbf{0}$, i.e. no additive Gaussian noise, 2. \textbf{Baseline Defense:} $\bfB$ in Eq.~\eqref{eq:mean_update} is restricted to be Isotropic Gaussian, 3. \textbf{Defense:} $\bfB$ in Eq.~\eqref{eq:mean_update} can be arbitrarily shaped. We use Algorithm~\ref{alg:meta_mean} to compute the defense parameter $\bfS = \bfB\bfB^\top$ for the latter 2 learning algorithms by training on 10 randomly chosen Gaussians drawn from standard Gaussian prior for the mean and standard Inverse Wishart prior for the covariance. We report the average test performance on 50 Gaussian distributions drawn from the same prior (see Figure~\ref{fig:mean_estimation}).

% \maryam{Let's see if we can interpret the $B$ obtained by the algorithm, and also examine it in the numerical experiments, looking at the eigenvectors of $B$.}

% \maryam{add more details about experiments and the baseline.}

\subsection{Image Classification}
\vskip -0.15in
% \avi{Dj, have a look at this.}
% We consider binary classification tasks on multiple datasets : (i) MNIST:1 vs 7 \citep{deng2012mnist}, (ii) FashionMNIST: Sandals vs Ankle Boot \citep{xiao2017fashion}, (iii) CIFAR-10: Cats vs Dogs \citep{krizhevsky2009learning}.
We consider binary classification tasks on multiple datasets : (i) MNIST \citep{deng2012mnist}, (ii) FashionMNIST \citep{xiao2017fashion}, (iii) CIFAR-10\citep{krizhevsky2009learning}. Detailed dataset descriptions and preprocessing steps can be found in Appendix~\ref{appendix:experiment}.

% Our experiments aim to address the key questions:
% \begin{enumerate}
%     \item How does the certificate of robustness vary as we change the values of  (a) $\epsilon$, (b) $\eta$, (c) $\sigma$? What factors make learning algorithms more vulnerable?
%     % \item Study the performance of the learning algorithm in absence of adversary by plotting the test loss upon training on the benign dataset as the following are varied (a) $\eta$, (b) $\sigma$.
%     \item How can we use the setup in \eqref{eq:main_formulation} to trade off performance in the absence of an adversary and the certificate of robustness to choose the best hyperparameters for a robust learning algorithm?
%     \item How does the test performance of our learning algorithm compare on various dynamic adversarial attacks when using the optimal defense obtained from our setup in \eqref{eq:main_formulation} vs.\ a poorly chosen defense?
%     \item Can we empirically demonstrate that our formulation allows us to compute a certificate that provides an upper bound for various types of attacks, showing that our methodology is effective without assuming specific attack types?
% \end{enumerate}

In our experiments, we choose $\Ptarget$ in Eq.~\eqref{eq:adv_classification} to be the same as $\Pdata$, i.e., the adversary's objective is to make the model perform poorly on the benign data distribution.

To learn the hyperparameters for a robust online learning algorithm (Eq.~\eqref{eq:hinge_online_learning}), we adopt the Meta learning setup presented in Eq.~\eqref{eq:main_formulation}. We choose $\mathcal{P}$ as a set of binary classification datasets on label pairs different from the label pair the online algorithm receives data from. For example, we use data from MNIST: 4 vs 9, 5 vs 8, 3 vs 8, 0 vs 6 to construct the objective in Eq.~\eqref{eq:main_formulation} and then test the performance of the onlne learning algorithm with the learnt hyperparameters on MNIST: 1 vs 7 (see Appendix~\ref{appendix:experiment} for more details). We compute certificates for different fractions of the training data to be corrupted $\epsilon \in \{1 \%, 2\%, 3\%, 4\%, 5\%\}$, and vary $\eta, \sigma$ for various values of $\eta$ and $\sigma$ by performing a grid search over $\eta \in \{5\times10^{-5}, 1\times10^{-4}, 5\times10^{-4}, 1\times10^{-3}, 5\times10^{-3}\}$ and $\sigma \in \{3\times10^{-3}, 6\times10^{-3}, 1\times10^{-2}, 3\times10^{-2}, 6\times10^{-2}\}$. We tune the hyperparameter $\kappa$ in Eq.~\eqref{eq:main_formulation} to trade off benign loss and robustness certificate by the generalization test loss on a held-out validation set.
In Figure~\ref{fig:classification} we plot the certificates of robustness corresponding to hyperparameters $(\eta, \sigma)$ pairs having smallest objective values in Eq.~\eqref{eq:main_formulation}.
% and an arbitrarily chosen hyperparameter.

Additionally, we evaluate the learning algorithm against three types of attacks:
1. \textbf{fgsm:} The attacker, having access to model parameters at each time step, initializes a random $\mathbf{z}_{\text{adv}}$, computes the gradient of $\bfzadv$ on the adversarial loss, takes a gradient ascent step on $\bfzadv$ and projects the result onto the unit $\ell_2$ ball, 2. \textbf{pgd:} similar to fgsm, but instead of a single big gradient step, the attacker takes multiple small steps and projects onto the unit $\ell_2$ ball, 3. \textbf{label flip:} the attacker picks an arbitrary data point and flips its label. In Figure~\ref{fig:classification}, we plot the best attack (corresponding to largest prediction error) for each value of $\epsilon$. We observe  the following:
The test adversarial loss under different heuristic attacks is consistently lower than the computed robustness certificate.
% \begin{enumerate}[topsep=-1pt,itemsep=-1ex,partopsep=-1ex,parsep=1ex]
%     \item  The test adversarial loss under different heuristic attacks is consistently lower than the computed robustness certificate.
%     \item Suboptimal choices of regularization and learning rate can significantly increase the vulnerability of learning algorithms, as indicated by both the robustness certificate and the adversarial loss when training with these parameter values (red plots in Figure~\ref{fig:classification}).
%     \vskip -1in
% \end{enumerate}

\subsection{Reward Learning}
\vskip -0.15in
We conduct reward learning experiments using Helpsteer, an open-source helpfulness dataset that is used to align large language models to become more helpful, factually correct and coherent, while being adjustable in terms of the complexity and verbosity of its responses \citep{wang2023helpsteer, dong2023steerlm}. Since the datasets of human feedback, both for open source and closed sourced models, are typically composed of users `in the wild` using the model, there is potential for adversaries to introduce poisoning. This can lead to the learned reward model favoring specific demographic groups, political entities or unscientific points of view, eventually leading to bad and potentially harmful experiences for users of language models fine tuned on the learned reward.\looseness=-1

To apply our framework to this problem, we consider a simple reward model. Given a (prompt, response) pair, we extract a feature representation $\bfx$ using a pretrained BERT model, and our reward model parameterized by $\bftheta$, predicts the reward as $\bftheta^\top \bfx$. Given the score $y$ (normalised to fall within the range $[-1,1]$) on a particular attribute (say helpfulness), the learning algorithm proceeds to learn the reward model by performing gradient descent on the regularized hinge loss objective Eq.~\eqref{eq:hinge}. We follow the similar hyperprameter search space as the image classification example using the meta learning setup in Eq.~\eqref{eq:main_formulation}. The results for helpfulness and correctness are presented in Figure~\ref{fig:llm}. More details are deferred to Appendix~\ref{appendix:experiment}.

% while findings related to other attributes are included in Appendix~.
% Types of plots to be added:
% \begin{enumerate}
%     \item x-axis : vary $\epsilon$, y-axis: value of certificate. Multiple plots for various $\eta, \sigma$ pairs. Also plot the corresponding benign losses. The benign losses should be horizontal lines as they don't depend on $\epsilon$. It would be good to have 2 sets of plots, one with $\sigma$ held fixed, the other with $\eta$ held fixed to study the effect of the other hyperparamter which is varied.
%     \item Try few attacks : greedy (use PGD, FGSM to obtain $\bfzadv_t$), label flip, random additive feature noise, and train them at different hyperparamter settings. Plot the test losses, and show certificate upper bounds the target loss for different attacks. Also demonstrate, no defense / defense with poor choice of $\eta, \sigma$ are vulnerable - resulting in larger test loss. Thereby supporting our methodology useful to create successful defenses.
% \end{enumerate}

\section{CONCLUSION AND FUTURE WORK}
\vskip -0.1in
    This paper presents a novel framework for computing certified bounds on the worst-case impacts of dynamic data poisoning adversaries in machine learning. This framework lays the groundwork for developing robust algorithms, which we demonstrate for mean estimation and linear classifiers. Extending this framework to efficient algorithms for more general learning setups, particularly deep learning setups that use human feedback with potential for malicious feedback, would be a compelling direction of future work. Furthermore, considering alternative strategies for computing the bound in \eqref{eq:cert_general}, particularly ones driven by AI advances, would be a promising approach for scaling. Recent work has demonstrated that AI systems can be used to discover Lyapunov functions for dynamical systems \citep{alfarano2023discovering}, indicating that AI driven approaches could hold promise in this setting.

\newpage
\newpage
\bibliography{references}
\bibliographystyle{plainnat}
\newpage
\appendix
% \twocolumn[
\onecolumn

\section{Notations}\label{appendix:notation}
\begin{table}[h!]\label{tab:notation}
\centering
\footnotesize
\begin{tabular}{lll}
\toprule
Notation & Interpretation & Belongs to \\
\midrule
   \vphantom{$F_\phi^d$}$\bftheta$  & Model Parameters & $\Thetaspace$ \\
   \vphantom{$F_\phi^d$}$\bfphi$ & Hyper-parameters & $\Phispace$ \\
   % \hline
   % $\bfw$ & Gaussian noise injected into learning algorithm  & $\Wspace$ \\
   \vphantom{$F_\phi^d$}$\bfz$ & Data point & $\Zspace$ \\
   \vphantom{$F_\phi^d$}$F_{\phi}$ & Update rule & $\Thetaspace \times \Zspace \mapsto \Thetaspace$ \\
   \vphantom{$F_\phi^d$}$\bfzadv$ & Adversarial data point & $\mathcal{A}$\\
   \vphantom{$F_\phi^d$}$\ell_{\rm adv}$ & Adversarial objective fn. & $\Thetaspace \mapsto \mathbb{R}$ \\
   \vphantom{$F_\phi^d$}$\Pdata$ & Benign Data Dist. & $\mathcal{P}[\bfZ]$ \\
   % \vphantom{$F_\phi^d$}$\Ptarget$ & Adversary's target Dist. & $\mathcal{P}[\bfZ]$ \\
   \vphantom{$F_\phi^d$}$\Pi(\cdot| \bftheta, \bfzadv)$ & State Transition Kernel & $\Thetaspace \!\times\! \Zspace \mapsto \mathcal{P}[\Thetaspace]$\!\!\\
   \bottomrule
\end{tabular}\caption{Notation.}
% \maryam{if we don't use $P^{\rm target}$ we can remove that line}}
\vskip -0.1in
\end{table}
\section{Proofs}\label{appendix:proofs}
\certificate*
\begin{proof}
The primal problem is defined below:
\begin{align*}
    &\sup_{\substack{d_{\pi} \in \mathcal{P}[\Thetaspace] \\ \pi \in \mathcal{P}[\Thetaspace \times \bfZ]}}  \quad \ExP{\bftheta \sim d_{\pi}(\cdot)}{\ell_{\textup{adv}}\br{\bftheta}}, \quad \text{subject to } \\
       & \ExP{\substack{\bftheta \sim d_{\pi}(\cdot) \\ \bfzadv \sim \pi(\cdot|\bftheta)}}{\Pi(\bftheta^\prime | \bftheta, \bfzadv)} =  d_{\pi}(\bftheta^\prime) \quad \forall \bftheta^\prime \in \Theta.
\end{align*}
For a given function $\lambda : \Thetaspace \mapsto \R$, the Lagrangian
% function \maryam{you should say the Lagrangian, not the dual function}
can be written as:
\begin{align*}
    \mathcal{L}(d_{\pi}, \pi, \lambda) &= \ExP{\bftheta \sim d_{\pi}(\cdot)}{\ell_{\textup{adv}}\br{\bftheta}} + \int_{\bftheta^\prime \in \Thetaspace} \lambda(\bftheta^\prime)\bigg[ \ExP{\substack{\bftheta \sim d_{\pi}(\cdot) \\ \bfzadv \sim \pi(\cdot|\bftheta)}}{\Pi(\bftheta^\prime | \bftheta, \bfzadv)} - d_{\pi}(\bftheta^\prime)\bigg] d\bftheta^\prime \\
    % &= \ExP{\substack{\bftheta \sim d_{\pi}(\cdot) \\ \bfzadv \sim \pi(\cdot|\bftheta)}}{\ell_{\textup{adv}}\br{\bftheta}} + \ExP{\substack{\bftheta \sim d_{\pi}(\cdot) \\ \bfzadv \sim \pi(\cdot|\bftheta)}}{\ExP{\bftheta^\prime \sim \Pi(\cdot | \bftheta, \bfzadv)}{\lambda(\bftheta^\prime) }} - \ExP{\substack{\bftheta^\prime \sim d_{\pi}(\cdot) \\ \bfzadv \sim \pi(\cdot|\bftheta^\prime)}}{\lambda(\bftheta^\prime)}
    &= \ExP{\bftheta \sim d_{\pi}(\cdot)}{\ell_{\textup{adv}}\br{\bftheta}} + \ExP{\substack{\bftheta \sim d_{\pi}(\cdot) \\ \bfzadv \sim \pi(\cdot|\bftheta)}}{\ExP{\bftheta^\prime \sim \Pi(\cdot | \bftheta, \bfzadv)}{\lambda(\bftheta^\prime) }} - \ExP{\bftheta^\prime \sim d_{\pi}(\cdot)}{\lambda(\bftheta^\prime)}\\
    &= \ExP{\substack{\bftheta \sim d_{\pi}(\cdot) \\ \bfzadv \sim \pi(\cdot|\bftheta)}}{\ExP{\bftheta^\prime \sim \Pi(\cdot | \bftheta, \bfzadv)}{\lambda(\bftheta^\prime) } + \ell_{\textup{adv}}\br{\bftheta} - \lambda(\bftheta)},
\end{align*}
which serves as an upper bound on the primal objective. Note that the value of the Lagrangian for a given $\lambda$ depends on the expectation of $\ExP{\bftheta^\prime \sim \Pi(\cdot | \bftheta, \bfzadv)}{\lambda(\bftheta^\prime) } + \ell_{\textup{adv}}\br{\bftheta} - \lambda(\bftheta)$ over the joint distribution of $\bftheta \times \bfzadv$. Since $\Thetaspace$ and $\mathcal{A}$ are compact, the supremum for a given $\lambda$ occurs for the distribution placing all its mass on the point $\bftheta, \bfzadv$ maximizing: $\ExP{\bftheta^\prime \sim \Pi(\cdot | \bftheta, \bfzadv)}{\lambda(\bftheta^\prime) } + \ell_{\textup{adv}}\br{\bftheta} - \lambda(\bftheta)$.

Therefore, for any choice of $\lambda$ and any feasible $d_{\pi}$, we have:
\begin{align*}
    \sup_{\substack{d_{\pi} \in \mathcal{P}[\Thetaspace] \\ \pi \in \mathcal{P}[\Thetaspace \times \bfZ]}} \mathcal{L}(d_{\pi}, \pi, \lambda) =  \sup_{\substack{\bftheta \in \Thetaspace \\ \bfzadv \in \Aspace}} \quad \ExP{\bftheta^\prime \sim \Pi(\cdot | \bftheta, \bfzadv)}{\lambda\br{\bftheta^\prime}} + \ell_{\textup{adv}}\br{\bftheta}-\lambda\br{\bftheta}.
\end{align*}
This completes the proof.
\end{proof}
\section{Mean Estimation}\label{appendix:mean_estimation}
% \textcolor{magenta}{rewrite the optimization problem as we did in 578: write out the objective explicitly, and write the domain constraints as constraints in the optimization problem, let's see what that looks like---it will be easier to parse and see what the dependence of objectives and constraints on the variables are. Also see if you can simplify the terms further. put the negative sign up front or see if you can cancel out the negatives. Else it looks like everything was dumped in the problem with no attempt to make it readable and could bother reviewers.
\meancertificate*
\begin{proof}
We can write the learning algorithm in Eq.~\eqref{eq:dyn} for the case of mean estimation as follows:
\begin{align*}
    \bftheta_{t+1} = F\br{\bftheta_t, \bfz_t} + \eta \bfB \bfw_t,
\end{align*}
where $F\br{\bftheta, \bfz}=\bftheta\br{1-\eta} + \eta\bfz$, which is a linear transformation of $\bftheta$ followed by additive Gaussian noise.

The transition distribution for the parameter is given by:
% \[\mathbb{P}_{\bfS, \Pdata, \bfzadv}\br{\bftheta^\prime|\bftheta} = \epsilon \mathcal{N}\br{\bftheta^\prime|\bftheta-\eta\br{\bftheta-\bfzadv}, \eta^2\bfS} + (1-\epsilon)\ExP{\bfz \sim \Pdata}{ \mathcal{N}\br{\bftheta^\prime|\bftheta-\eta\br{\bftheta-\bfz}, \eta^2\bfS}}\]
\begin{align}
\mathbb{P}_{\bfS, \Pdata, \bfzadv}\br{\bftheta^\prime|\bftheta} = \epsilon \mathcal{N}\br{\bftheta^\prime|F(\bftheta, \bfzadv), \eta^2\bfS} + (1-\epsilon)\ExP{\bfz \sim \Pdata}{ \mathcal{N}\br{\bftheta^\prime|F(\bftheta, \bfz), \eta^2\bfS}} \label{eq:transition_mean}
\end{align}
which is a Gaussian distribution whose mean depends linearly on $\bftheta$ and $\bfzadv$.

 Then, we have from Eq.~\eqref{eq:certificate} that the certified bound on the adversarial objective is given by:
% \avi{Typo: First 2 terms below should be $\lambda(\theta')$ not $\lambda(\theta)$?}
\begin{subequations}
\begin{align}
&\sup_{\substack{\bfzadv \in \mathcal{A} \\ \bftheta}} \epsilon\ExP{\bftheta^\prime \sim \mathcal{N}\br{F\br{\bftheta, \bfzadv}, \eta^2 \bfS}}{\lambda\br{\bftheta^\prime}}+(1-\epsilon)\ExP{\bfz \sim \Pdata}{\ExP{\bftheta^\prime \sim \mathcal{N}\br{F\br{\bftheta, \bfz}, \eta^2 \bfS}}{\lambda\br{\bftheta^\prime}}}-\lambda\br{\bftheta}+\ell_{\textrm{adv}}\br{\bftheta} \\
&\text{We choose $\lambda\br{\bftheta} = \bftheta^\top \bfA \bftheta + \bftheta^\top \bfb$ to be a quadratic function. Then we have:}\nonumber\\
&=\sup_{\substack{\bfzadv \in \mathcal{A} \\ \bftheta}} \epsilon\br{ \lambda\br{F\br{\bftheta, \bfzadv}} + \eta^2 \inner{\nabla^2 \lambda\br{0}}{\bfS}} + (1-\epsilon)\ExP{\bfz \sim \Pdata}{ \lambda\br{F\br{\bftheta, \bfz}} + \eta^2 \inner{\nabla^2 \lambda\br{0}}{\bfS}}-\lambda\br{\bftheta}+\ell_{\textrm{adv}}\br{\bftheta} \\
&=\sup_{\substack{\bfzadv : \|\bfzadv - \mu\|_2^2 \leq r \\ \bftheta}}  -\left\|\bmat{\bftheta \\ \bfzadv}\right\|^2_{\bfE^{-1}} + \bmat{\bftheta \\ \bfzadv}^\top \bmat{2(1 - \epsilon)\eta(1 - \eta)\bfA \mu - 2 \mu - \eta \bfb \\ \epsilon \eta \bfb}  \nonumber\\
&+ (1-\epsilon)(\eta^2 \Trace(\bfSigma \bfA) + \eta^2 \mu^\top \bfA \mu + \eta \bfb^\top \mu) + \mu^\top \mu, \\
&\text{where $\bfE = \bmat{ (1 - (1 - \eta)^2)\bfA - \bfI & -\eta\epsilon(1 - \eta)\bfA \\ -\eta\epsilon(1 - \eta)\bfA &  -\epsilon\eta^2 \bfA}$,}\nonumber\\
&\text{The dual function of this supremum (with dual variable $\nu$) can be written as: } \nonumber\\
&=\inf_{\nu \geq 0} \sup_{\substack{\bfzadv \\ \bftheta}}  -\left\|\bmat{\bftheta \\ \bfzadv}\right\|^2_{\bfD^{-1}} + \bmat{\bftheta \\ \bfzadv}^\top \bmat{2(1 - \epsilon)\eta(1 - \eta)\bfA \mu - 2 \mu - \eta \bfb \\ \epsilon \eta \bfb + 2 \nu \mu}  \nonumber\\
&+ (1-\epsilon)(\eta^2 \Trace(\bfSigma \bfA) + \eta^2 \mu^\top \bfA \mu + \eta \bfb^\top \mu) + \mu^\top \mu + \nu(r - \mu^\top \mu) \\
&\text{where $\bfD = \bmat{ (1 - (1 - \eta)^2)\bfA - \bfI & -\eta\epsilon(1 - \eta)\bfA \\ -\eta\epsilon(1 - \eta)\bfA &  -\epsilon\eta^2 \bfA - \nu\bfI}$.}\nonumber\\
&\text{The inner supremum is a quadratic expression in $\bfzadv, \bftheta$. A finite supremum exists if the} \nonumber \\ &\text{Hessian of the expression is negative semifdefinite. Plugging in the tractable maximizer}\nonumber\\ &\text{of the quadratic, we get:} \nonumber\\
& \inf_{\nu \geq 0}\frac{1}{4} \left\|\bmat{2(1 - \epsilon)\eta(1 - \eta)\bfA \mu - 2 \mu - \eta \bfb \\ \epsilon \eta \bfb + 2 \nu \mu}\right\|^2_{\bfD} + (1-\epsilon)(\eta^2 \Trace(\bfSigma \bfA) + \eta^2 \mu^\top \bfA \mu + \eta \bfb^\top \mu) \nonumber \\& + \mu^\top \mu + \eta^2 \Trace(\bfA\bfS) + \nu(r - \mu^\top \mu)   \; \text{such that} \; \bfD \succeq 0.
\end{align}
\end{subequations}
This completes the proof.
% \textbf{Certified Bound for a fixed leaner:} Define $\bfD_1 = ((1 - \eta)^2 - 1)\bfA + \bfI, \bfD_2 = 2\eta\epsilon(1 - \eta)\bfA, \bfd_3 = 2(1 - \epsilon)\eta(1 - \eta)\bfA \mu - 2 \mu - \eta \bfb, d_4 = (1-\epsilon)(\eta^2 \Trace(\bfSigma \bfA) + \eta^2 \mu^\top \bfA \mu + \eta \bfb^\top \mu) + \mu^\top \mu + \eta^2 \Trace(\bfA\bfS)$. The bound on the adversarial objective is given by:

% \begin{align}
%     \sup_{\substack{\bfzadv \in \mathcal{A} \\ \bftheta}} \bftheta^\top \bfD_1 \bftheta + \epsilon \eta^2 {\bfzadv}^\top \bfA \bfzadv + \bftheta^\top \bfD_2 \bfzadv + {\bftheta}^\top \bfd_3 + \epsilon \eta {\bfzadv}^\top \bfb + d_4.
% \end{align}
% The inner supremum is a quadratic expression in $\bfzadv, \bftheta$. A finite supremum exists if the Hessian of the expression is negative semifdefinite.
% The dual function of this supremum (with dual variable $\nu$) can be written as in Eq.~\eqref{eq:dualfunction}.
\end{proof}

\begin{lem}
The stationary distribution in the absence of adversary in Eq.~\eqref{eq:stationary} for the mean estimation problem for $\Pdata = \mathcal{N}(\mu, \bfSigma)$ takes the form:
\begin{align*}
   \mathbb{P}\br{\bfS, \Pdata} = \mathcal{N}(\mu, \eta^2 \bfS).
\end{align*}
\end{lem}
\begin{proof}
The stationary distribution is tractable in this case. Recall from Eq.~\eqref{eq:transition_mean}, setting $\epsilon=0$, the transition distribution conditioned on $\bftheta$ is a Gaussian whose mean is linear in $\bftheta$. Therefore the stationary distribution:
\begin{align*}
    \ExP{\bftheta \sim \mathbb{P}}{\mathbb{P}_{\bfS, \Pdata}\br{\bftheta^\prime|\bftheta}},
\end{align*}
will be a Gaussian distribution as a sum of Gaussians is also a Gaussian. Let us assume the distribution has mean $\bfm$. Comparing the means we have:
\begin{align*}
    \bfm (1 - \eta) + \eta \mu &= \bfm\\
    \implies \bfm = \mu.
\end{align*}
Moreover, $\mathbb{P}_{\bfS, \Pdata}\br{\bftheta^\prime|\bftheta}$ is a Gaussian with covariance $\eta^2 \bfS$ for all $\bftheta$. Hence the expectation over $\mathbb{P}$ also has covariance $\eta^2 \bfS$. This concludes the proof.
\end{proof}
\begin{lem}
The loss at stationarity of the learning dynamics in the absence of an adversary for the mean estimation problem for $\Pdata = \mathcal{N}(\mu, \bfSigma)$ is given by:
\begin{align}
    \ExP{\bftheta \sim \mathbb{P}\br{\bfS, \Pdata}}{\ell\br{\bftheta}} = \eta^2 \Trace(\bfS).
\end{align}
\end{lem}
\begin{proof}
\begin{align*}
    &\ExP{\theta \sim \mathcal{N}(\mu, \eta^2 \bfS)}{\|\theta - \mu\|_2^2}\\
    &= \ExP{\theta \sim \mathcal{N}(0, \eta^2 \bfS)}{\|\theta\|_2^2}\\
    &= \eta^2 \Trace(\bfS).
\end{align*}
\end{proof}
% \paragraph{Defense for a single data distribution}
% % \textbf{Parameters for robust learning algorithm:}
% An upper bound on the objective in Eq. \eqref{eq:mean_certificate} for a single given $\Pdata$ can thus be written as the following minimization problem:
% % \begin{align}
% %     \inf_{\substack{\bfS \in \mathbb{S}^d_+ \\ \bfA \in \mathbb{S}^d, \bfb \in \R^d \\ \nu \leq 0}}\qquad &\eta^2 \Trace(\bfS)  + \kappa  g(\bfA, \bfb,  \nu, \bfS, \mu, \bfSigma) \label{eq:mean_defense}
% % \end{align}
% \begin{align}
%     \inf_{\substack{\bfS \in \mathbb{S}^d_+ \\ \bfA \in \mathbb{S}^d, \bfb \in \R^d \\ \nu \geq 0}}\qquad &  g(\bfA, \bfb,  \nu, \bfS, \mu, \bfSigma) \label{eq:mean_certificate_min}
% \end{align}
\begin{remark}
We use CVXPY \citep{diamond2016cvxpy} to solve the optimization problems in Algorithm~\ref{alg:meta_mean}.
\end{remark}

\section{Binary Classification}\label{appendix:binary_classification}

% \begin{lem}
% If $\|\bftheta_0\|_2 \leq \frac{\sqrt{d}}{\sigma}$, then for all $t > 0$, $\|\bftheta_t\|_2 \leq \frac{\sqrt{d}}{\sigma}$.
% \end{lem}
% \begin{proof}
%     Use Induction and traingle inequality.
% \end{proof}
\begin{lem}
If $\|\bftheta_0\|_{2} \leq \frac{1}{\sigma}$, then for all $t > 0$, $\|\bftheta_t\|_{2} \leq \frac{1}{\sigma}$.
\end{lem}
\begin{proof}\label{lem:QPdual}
    Use Induction and triangle inequality.
\end{proof}
\begin{lem}
Consider the primal problem (here $\bfQ \succ 0$):
\begin{align*}
    &\sup_{\bfx, \bfy} -\bfx^\top \bfQ \bfx + \bfp_1^\top \bfx + \bfp_2^\top \bfy \\
    & \text{such that}\\
    & \bfA_{1i}^\top \bfx + \bfA_{2i}^\top \bfy \succeq \bfc_i \; \forall i \in [m].
\end{align*}
Its dual is the following oprimization problem:
\begin{align*}
    &\inf_{\nu_{1}, \ldots, \nu_{m} \succeq 0} \frac{1}{4} \|\bfp_1 + \sum_{i \in [m]} \bfA_{1i}\nu_i\|^2_{\bfQ^{-1}} - \sum_{i \in [m]} \nu_{i}^\top \bfc_i\\
    &\text{such that}\\
    & \bfp_2 + \sum_{i \in [m]} \bfA_{2i}\nu_i = \bf0.
\end{align*}
\end{lem}
\begin{proof}
    We write the primal objective's dual function with Lagrange parameters $\nu_{1}, \ldots, \nu_{m} \succeq 0$ as follows:
\begin{align*}
    \sup_{\bfx, \bfy} -\bfx^\top \bfQ \bfx + (\bfp_1 + \sum_{i \in [m]} \bfA_{1i}\nu_i)^\top \bfx + (\bfp_2 + \sum_{i \in [m]} \bfA_{2i}\nu_i)^\top \bfy - \sum_{i \in [m]} \nu_{i}^\top \bfc_i.
\end{align*}
The supremum is maximized for:
\begin{align*}
    \bfx^\ast = \frac{1}{2} \bfQ^{-1} (\bfp_1 + \sum_{i \in [m]} \bfA_{1i}\nu_i),
\end{align*}
and since we don't have a lower bound on $\bfy \succeq \mathbf{0}$, we need $\bfp_2 + \sum_{i \in [m]} \bfA_{2i}\nu_i = \bf0$.

Plugging these in, the dual function completes the proof.
\end{proof}
% where $\Pi$ is a projection.
% For example $\Pi(x) = x \; \text{if} \; \|x\|_2 \leq \zeta,\; \text{else} \; \zeta x/\|x\|_2$.
% Thus the update is given by:
% % let the update rule be defined by:
% \begin{align}
%     \bftheta_{t+1} =  F(\bftheta_t, \bfz) + \eta \bfB \bfw_t \quad (\text{where} \; \bfw_t \sim \mathcal{N}(0, \bfI)).
% \end{align}

% One can view $\mathbb{I}[\bftheta^\top \bfz \leq 1]$ as a measure of the goodness of the current prediction, and thus controls by how much the parameter should be updated based on the current sample.
% \djcomment{Why is the }
\classificationhingecertificate*
\begin{proof}
Since $\epsilon$ fraction of the samples are corrupted by the adversary, the transition distribution conditioned on $\bfzadv$, the benign distribution $\Pdata$ and the defense parameter $\sigma$ is given by:
% \[\mathbb{P}_{\bfS, \Pdata, \bfzadv}\br{\bftheta^\prime|\bftheta} = \epsilon \mathcal{N}\br{\bftheta^\prime|F(\bftheta, \bfzadv), \eta^2\bfS} + (1-\epsilon)\ExP{\bfz \sim \Pdata}{ \mathcal{N}\br{\bftheta^\prime|F(\bftheta, \bfz), \eta^2\bfS}},\]
\begin{align*}
\mathbb{P}_{\sigma, \Pdata, \bfzadv}\br{\bftheta^\prime|\bftheta} &= \epsilon F(\bftheta, \bfzadv) + (1 - \epsilon) \ExP{\bfz \sim \Pdata} {F(\bftheta, \bfz)}\\
&= \epsilon ((1 - \sigma \eta) \bftheta +  \eta \mathbb{I}[\bftheta^\top \bfzadv \leq 1] \bfzadv) + (1 - \epsilon) ((1 - \sigma \eta)\bftheta + \eta \ExP{\bfz \sim \Pdata} {\mathbb{I}[\bftheta^\top \bfz \leq 1] \bfz}).
\end{align*}
We wish to analyse the adversarial loss at stationarity. We consider the following 2 cases at stationarity. Consider $\Thetaspace_1, \mathcal{A}_1$ as the space such that
(i) $\mathbb{I}[\bftheta^\top \bfzadv \leq 1] = 0$, the transition distribution at stationarity looks like:
\begin{align*}
\mathbb{P}_{\sigma, \Pdata, \bfzadv}\br{\bftheta^\prime|\bftheta} &= \epsilon F(\bftheta, \bfzadv) + (1 - \epsilon) \ExP{\bfz \sim \Pdata} {F(\bftheta, \bfz)}\\
&= \epsilon ((1 - \sigma \eta) \bftheta +  \eta \mathbb{I}[\bftheta^\top \bfzadv \leq 1] \bfzadv) + (1 - \epsilon) ((1 - \sigma \eta)\bftheta + \eta \ExP{\bfz \sim \Pdata} {\mathbb{I}[\bftheta^\top \bfz \leq 1] \bfz})\\
&= (1 - \frac{\sigma}{1 - \epsilon} (1 - \epsilon) \eta) \bftheta +  (1 - \epsilon)\eta \ExP{\bfz \sim \Pdata} {\mathbb{I}[\bftheta^\top \bfz \leq 1] \bfz}
\end{align*}
This can be interpreted as the stationary distribution in the absence of an adversary with learning rate $(1 - \epsilon) \eta$ and regularization $\frac{\sigma}{(1 - \epsilon)}$.

The other case is $\Thetaspace_2, \mathcal{A}_2$ such that (ii) $\mathbb{I}[\bftheta^\top \bfzadv \leq 1] = 1$.

Note that $\{\Thetaspace_1 \times \mathcal{A}_1\} \cup \{\Thetaspace_2 \times \mathcal{A}_2\}= \Thetaspace \times \mathcal{A}$.
We can treat both these cases separately in our original problem Eq.~\ref{eq:lp_inf} before going to the formulation in Eq.~\ref{eq:cert_general}. Thus we aim to find a certificate for each of these cases and we do so via the formulation in Eq.~\ref{eq:cert_general} and take the max of these upper bounds.

% \avi{Yet to modify the writing below, but hope the above explanation sounds reasomable as to why the max outside works.}
% which is a Gaussian distribution with covariance $\eta^2 \bfS$.

Choosing $\lambda\br{\bftheta} = \bftheta^\top \bfA \bftheta + \bftheta^\top \bfb$ to be a quadratic function, we derive the certificate for a fixed $\sigma$:
\begin{align*}
    &\textrm{OPT} = \sup_{\substack{\bfzadv \in \mathcal{A} \\ \bftheta : \|\bftheta\|_{2} \leq \frac{1}{\sigma}}} \epsilon\br{ \lambda\br{F\br{\bftheta, \bfzadv}}} + (1-\epsilon)\ExP{\bfz \sim \Pdata}{ \lambda\br{F\br{\bftheta, \bfz}}} -\lambda\br{\bftheta}+\ell_{\textrm{adv}}\br{\bftheta} \\
     &= \sup_{\substack{\bfzadv \in \mathcal{A} \\ \bftheta : \|\bftheta\|_{2} \leq \frac{1}{\sigma}}} [(1 - \sigma \eta)^2 - 1] \bftheta^\top \bfA \bftheta + \epsilon \eta^2 \mathbb{I}[\bftheta^\top \bfzadv \leq 1] {\bfzadv}^\top \bfA \bfzadv +   2 \epsilon (1 - \sigma \eta) \eta \mathbb{I}[\bftheta^\top \bfzadv \leq 1] \bftheta^\top \bfA  \bfzadv  \\ & + \epsilon \eta \mathbb{I}[\bftheta^\top \bfzadv \leq 1] \bfb^\top\bfzadv] + (1 - \epsilon)\ExP{\bfz \sim \Pdata}{\eta^2 \mathbb{I}[\bftheta^\top \bfz \leq 1] \bfz^\top \bfA \bfz + 2 \eta (1 - \sigma \eta) \mathbb{I}[\bftheta^\top \bfz \leq 1] \bftheta^\top \bfA \bfz + \eta \mathbb{I}[\bftheta^\top \bfz \leq 1] \bfb^\top  \bfz} \\& - \sigma \eta \bfb^\top \bftheta + \ExP{z \sim \Pdata}{\max \{0, 1 - \bftheta^\top \bfz\}}\\
     &\text{(Using sample average approximation for $\Pdata$ with data points from the training data set we get:)}\\
      &= \sup_{\substack{\bfzadv \in \mathcal{A} \\ \bftheta : \|\bftheta\|_{2} \leq \frac{1}{\sigma} \\ q_i \in \{0,1 \} \forall i \in [N]}} [(1 - \sigma \eta)^2 - 1] \bftheta^\top \bfA \bftheta + \epsilon \eta^2 \mathbb{I}[\bftheta^\top \bfzadv \leq 1] {\bfzadv}^\top \bfA \bfzadv +   2 \epsilon (1 - \sigma \eta) \eta \mathbb{I}[\bftheta^\top \bfzadv \leq 1] \bftheta^\top \bfA  \bfzadv \\& + \epsilon \eta \mathbb{I}[\bftheta^\top \bfzadv \leq 1] \bfb^\top\bfzadv  + (1 - \epsilon) \frac{1}{N}\sum_{i \in [N]} \bigg [\eta^2 q_i \bfz_i^\top \bfA \bfz_i + 2 \eta (1 - \sigma \eta) q_i \bftheta^\top \bfA \bfz_i + \eta q_i \bfb^\top \bfz_i\bigg] + \\&  \frac{1}{N}\sum_{i \in [N]} \bigg[q_i (1 - \bftheta^\top \bfz_i)\bigg] - \sigma \eta \bfb^\top \bftheta \\
      &\text{such that} \; 1 - \bftheta^\top \bfz_i \leq (1 + \frac{d}{\sigma})q_i, 1 - \bftheta^\top \bfz_i \geq -(1 + \frac{d}{\sigma})(1 - q_i) \forall i \in [N].
\end{align*}
To get rid of the indicator variable $\mathbb{I}[\bftheta^\top \bfzadv \leq 1]$, we can write the certificate as the maximum of 2 optimization problems, (i) $\OPT_1$ with $\mathbb{I}[\bftheta^\top \bfzadv \leq 1] = 0$ and, (ii) $\OPT_2$ with $\mathbb{I}[\bftheta^\top \bfzadv \leq 1] = 1$. Note that the optimization problem in $\OPT_1$ doesn't have $\bfzadv$ as a decision variable. We first focus on the relaxations on the optimization problem in $\OPT_2$ and a bound on the optimization problem in $\OPT_1$ can be obtained by dropping the terms in $\bfzadv$ from $\OPT_2$.

\begin{align*}
\OPT_2 = & \sup_{\substack{\bfzadv \in \mathcal{A} \\ \bftheta : \|\bftheta\|_{2} \leq \frac{1}{\sigma} \\ q_i \in \{0,1\} \forall i \in [N]}} [(1 - \sigma \eta)^2 - 1] \bftheta^\top \bfA \bftheta + \epsilon \eta^2 {\bfzadv}^\top \bfA \bfzadv +   2 \epsilon (1 - \sigma \eta) \eta \bftheta^\top \bfA  \bfzadv + \epsilon \eta \bfb^\top\bfzadv \\& + (1 - \epsilon) \frac{1}{N}\sum_{i \in [N]} \bigg [\eta^2 q_i \bfz_i^\top \bfA \bfz_i + 2 \eta (1 - \sigma \eta) q_i \bftheta^\top \bfA \bfz_i + \eta q_i \bfb^\top \bfz_i\bigg] +  \frac{1}{N}\sum_{i \in [N]} \bigg[q_i (1 - \bftheta^\top \bfz_i)\bigg] - \sigma \eta \bfb^\top \bftheta \\
      &\text{such that} \; \bftheta^\top \bfzadv \leq 1, 1 - \bftheta^\top \bfz_i \leq (1 + \frac{1}{\sigma})q_i, 1 - \bftheta^\top \bfz_i \geq -(1 + \frac{1}{\sigma})(1 - q_i) \forall i \in [N]\\
      &\text{(Relaxing integer variables $q_i$'s to continuous variables)}\\
      \leq & \sup_{\substack{\bfzadv \in \mathcal{A} \\ \bftheta : \|\bftheta\|_{2} \leq \frac{1}{\sigma} \\ q_i \in [0,1] \forall i \in [N]}} [(1 - \sigma \eta)^2 - 1] \bftheta^\top \bfA \bftheta + \epsilon \eta^2 {\bfzadv}^\top \bfA \bfzadv +   2 \epsilon (1 - \sigma \eta) \eta \bftheta^\top \bfA  \bfzadv + \epsilon \eta \bfb^\top\bfzadv \\& + (1 - \epsilon) \frac{1}{N}\sum_{i \in [N]} \bigg [\eta^2 q_i \bfz_i^\top \bfA \bfz_i + 2 \eta (1 - \sigma \eta) q_i \bftheta^\top \bfA \bfz_i + \eta q_i \bfb^\top \bfz_i\bigg] +  \frac{1}{N}\sum_{i \in [N]} \bigg[q_i (1 - \bftheta^\top \bfz_i)\bigg] - \sigma \eta \bfb^\top \bftheta \\
      &\text{such that} \; \bftheta^\top \bfzadv \leq 1, 1 - \bftheta^\top \bfz_i \leq (1 + \frac{1}{\sigma})q_i, 1 - \bftheta^\top \bfz_i \geq -(1 + \frac{1}{\sigma})(1 - q_i) \forall i \in [N].
\end{align*}

      Using McCormick relaxations for bilinear terms $q_i\bftheta$, we get:
\begin{align*}
      \OPT_2 \leq& \sup_{\substack{\bfzadv \in \mathcal{A}, \bftheta : \|\bftheta\|_2 \leq \frac{1}{\sigma}, \\ \bfq \in \R^N, \bfw \in \R^{dN}}} -\bmat{\bftheta \\ \bfzadv}^\top\bmat{[1 - (1 - \sigma \eta)^2] \bfA & -\epsilon (1 - \sigma \eta) \eta \bfA\\ -\epsilon (1 - \sigma \eta) \eta \bfA & -\epsilon \eta^2 \bfA} \bmat{\bftheta \\ \bfzadv} +
      \\& \bmat{-\sigma \eta \bfb \\ \epsilon \eta \bfb}^\top\bmat{\bftheta & \bfzadv} + \frac{1}{N} \sum_{1=1}^N \bmat{(1 - \epsilon) \eta^2 \bfz_i^\top \bfA \bfz_i + (1 - \epsilon) \eta \bfb^\top \bfz_i + 1 \\ 2(1 - \epsilon) \eta (1 - \sigma \eta) \bfA \bfz_i - \bfz_i}^\top\bmat{q_i & \bfw_i} \\
      &\text{such that} \\& \bfz_i^\top \bftheta + (1 + \frac{1}{\sigma}) q_i \geq 1 \; \forall i \in [N],
      \\& -\bfz_i^\top \bftheta -(1 + \frac{1}{\sigma}) q_i  \geq -(2 + \frac{1}{\sigma}) \; \forall i \in [N],
          \\& q_i \mathbf{1} / \sigma + \bfw_i \succeq \mathbf{0} \;\forall i \in [N], \\&
      -\bftheta - q_i \mathbf{1} / \sigma + \bfw_i \succeq  - \mathbf{1} / \sigma \;\forall i \in [N], \\&
       q_i \mathbf{1} / \sigma -\bfw_i \succeq \mathbf{0} \;\forall i \in [N], \\&
        \bftheta  - q_i \mathbf{1} / \sigma - \bfw_i \succeq -\mathbf{1} / \sigma\;\forall i \in [N],
      \\& \bfq \succeq \bf0,
      \\& -\bfq \succeq -\bf1,
      \\&1 - \bftheta^\top \bfzadv \geq 0.
      % \\&
      % -\bftheta \succeq -\mathbf{1} / \sigma,\\&
      % \bftheta \succeq -\mathbf{1} / \sigma,\\&
      % \bfzadv \in \mathcal{A}.
\end{align*}
% Note that the objective and constraints are quadratic only in $\bftheta, \bfzadv$, and linear in $\bfq, \bfw_1, \ldots, \bfw_N$.
% Writing the objective in a vectorized form we get:
% As long as the typical set $\mathcal{A}$ is linear ($\ell_{\infty}$ constraints) or quadratic ($\ell_2$ constraints) in $\bfzadv$, the above problem is a Quadratic Program.
% and we can write its dual as an infimum over various lagrangian parameters. Total number of lagrangian parameters to optimize over scales as $\mathcal{O}(Nd)$.

% Define affine functions in $\bfA, \bfb$:
% \begin{align*}
% &\bfs(\bfz_i, \bfA, \bfb) = \frac{1}{N} \bmat{\left((1 - \epsilon) \eta^2 \bfz_i^\top \bfA \bfz_i + (1 - \epsilon) \eta \bfb^\top \bfz_i + 1\right)\\\left( 2(1 - \epsilon) \eta (1 - \sigma \eta) \bfA \bfz_i - \bfz_i\right)} \in \R^{d + 1}.
% \end{align*}
From the Mccormick envelope constraints, and the constraints on $\bftheta$ and $\bfq$, we can derive that $\bfw_i \geq 0 \; \forall i \in [N]$.
Plugging $\mathcal{A} = \{\bfzadv | \|\bfzadv\|_{2} \leq 1\}$. Now we utilize Lemma~\ref{lem:QPdual} to write the dual problem as:
\begin{align*}
    & \inf_{\substack{\nu_1, \nu_2, \nu_7 \in \R^N_{+} \\ \nu_3, \nu_4, \nu_5, \nu_6 \in \R^{N \times d}_{+} \\ \nu_8, \nu_9, \nu_{10} \in \R_{+}}} \frac{1}{4}\|\bmat{-\sigma \eta \bfb + \sum_{i \in [N]} \left((\nu_{1i} - \nu_{2i}) \bfz_i - \nu_{4i} +\nu_{6i} \right) \\ \epsilon \eta \bfb} \|^2_{\bfD^{-1}}\\
      & -\nu_1^\top \mathbf{1} + (2 + \frac{1}{\sigma}) \nu_2^\top \mathbf{1} + \frac{(\nu_4 + \nu_6)^\top \bf1}{\sigma} + \mathbf{1}^\top \nu_7 + \frac{\nu_{8}}{\sigma^2} + 2 \nu_9 + \nu_{10}.\\
      & \text{such that} \\
      % & \bmat{(1 + \frac{d}{\sigma}) \nu_{1i} \\ \mathbf{0}} + \bmat{-(1 + \frac{d}{\sigma}) \nu_{2i} \\ 0} + \bmat{\frac{\mathbf{1}}{\sigma} & \mathbf{0} \\ \mathbf{0} & \bfI}^\top  \nu_{3i} + \bmat{-\frac{\mathbf{1}}{\sigma} & \mathbf{0} \\ \mathbf{0} & \bfI}^\top  \nu_{4i} +  \bmat{\frac{\mathbf{1}}{\sigma} & \mathbf{0} \\ \mathbf{0} & -\bfI}^\top  \nu_{5i} +
      % \\& \bmat{-\frac{\mathbf{1}}{\sigma} & \mathbf{0} \\ \mathbf{0} & -\bfI}^\top  \nu_{6i} + \bmat{\nu_7 - \nu_8 \\ \mathbf{0}} + \bfs(\bfz_i, \bfA, \bfb) = \mathbf{0} \;\forall i \in [N].\\
      & \bmat{(1 + \frac{1}{\sigma}) (\nu_{1i} - \nu_{2i})  - \frac{\mathbf{1}^\top(\nu_{4i} + \nu_{6i})}{\sigma} - \nu_{7i}\\ \nu_{3i} + \nu_{4i} - \nu_{5i} - \nu_{6i}} + \bfs(\bfz_i, \bfA, \bfb) \preceq \mathbf{0} \;\forall i \in [N].\\
      & \bfD \succeq 0
\end{align*}
where
\begin{align*}
    \bfD = \bmat{[1 - (1 - \sigma \eta)^2] \bfA + \nu_8 \bfI & -\epsilon (1 - \sigma \eta) \eta \bfA + \nu_9\\ -\epsilon (1 - \sigma \eta) \eta \bfA + \nu_9 & -\epsilon \eta^2 \bfA + \nu_{10}\bfI}.
\end{align*}
% \avi{Some thoughts:
% \begin{enumerate}
%     \item Consider the case $\epsilon=0$. We should expect certificate should equal the benign loss. If we restrict the class of $\lambda$, does there exist $\lambda$ such that $\ExP{\bfz \sim \Pdata}{ \lambda\br{F\br{\bftheta, \bfz}}} \leq \lambda\br{\bftheta} \forall \theta$ and equality holds for $\theta^\ast$ (which is the optimal parameter).
% \end{enumerate}}
% \djcomment{You need the $\bfD$ matrix to be PSD, otherwise the supremum is infinite. You should enforce this as a constraint in the infimum above.}\\
We use $\nu \succeq \mathbf{0}$ to compactly denote $\{\nu_1, \ldots, \nu_{10}\}$. Define the following affine functions in $\nu, \bfA, \bfb$:
\begin{align*}
    &\bfp(\bfb, \nu) = \frac{1}{2}\bmat{-\sigma \eta \bfb + \sum_{i \in [N]} \left((\nu_{1i} - \nu_{2i}) \bfz_i - \nu_{4i} +\nu_{6i} \right) \\ \epsilon \eta \bfb} \in \R^{2d},\\
    &\bfD(\bfA, \nu) = \bmat{[1 - (1 - \sigma \eta)^2] \bfA + \nu_8 \bfI & -\epsilon (1 - \sigma \eta) \eta \bfA + \nu_9\\ -\epsilon (1 - \sigma \eta) \eta \bfA + \nu_9 & -\epsilon \eta^2 \bfA + \nu_{10}\bfI} \in \mathbb{S}_{+}^{2d},\\
    % {2d \times 2d}\\
    &q(\nu) =-\nu_1^\top \mathbf{1} + (2 + \frac{1}{\sigma}) \nu_2^\top \mathbf{1} + \frac{(\nu_4 + \nu_6)^\top \bf1}{\sigma} + \mathbf{1}^\top \nu_7 + \frac{\nu_{8}}{\sigma^2} + 2 \nu_9 + \nu_{10} \in \R, \\
    &\bfr_i(\nu) = \bmat{(1 + \frac{1}{\sigma}) (\nu_{1i} - \nu_{2i})  - \frac{\mathbf{1}^\top(\nu_{4i} + \nu_{6i})}{\sigma} - \nu_{7i}\\ \nu_{3i} + \nu_{4i} - \nu_{5i} - \nu_{6i}} \in \R^{d + 1},\\
      &\bfs(\bfz_i, \bfA, \bfb) = \frac{1}{N} \bmat{\left((1 - \epsilon) \eta^2 \bfz_i^\top \bfA \bfz_i + (1 - \epsilon) \eta \bfb^\top \bfz_i + 1\right)\\\left( 2(1 - \epsilon) \eta (1 - \sigma \eta) \bfA \bfz_i - \bfz_i\right)} \in \R^{d + 1}.
\end{align*}
The certificate can thus be written compactly as:
\begin{align*}
    \OPT_2 \leq &\inf_{\nu \succeq \mathbf{0}} \|\bfp(\bfb, \nu)\|^2_{\bfD(\bfA, \nu)^{-1}} + q(\nu) \\
    &\text{such that} \\
    & \bfr_i(\nu) + \bfs(z_i, \bfA, \bfb) \preceq 0 \; \forall i \in [N], \bfD(\bfA, \nu) \succ \mathbf{0}.
\end{align*}
Similarly $\OPT_1$ can be upper bounded as:
\begin{align*}
    \OPT_1 \leq &\inf_{\nu \succeq \mathbf{0}} \|\bfp'(\bfb, \nu)\|^2_{\bfD'(\bfA, \nu)^{-1}} + q'(\nu) \\
    &\text{such that} \\
    & \bfr_i(\nu) + \bfs(z_i, \bfA, \bfb) \preceq 0 \; \forall i \in [N], \bfD'(\bfA, \nu) \succ \mathbf{0}.
\end{align*}
where:
\begin{align*}
    &\bfp'(\bfb, \nu) = \frac{1}{2}\bmat{-\sigma \eta \bfb + \sum_{i \in [N]} \left((\nu_{1i} - \nu_{2i}) \bfz_i - \nu_{4i} +\nu_{6i} \right)} \in \R^{d},\\
    &\bfD'(\bfA, \nu) = [1 - (1 - \sigma \eta)^2] \bfA + \nu_8 \bfI  \in \mathbb{S}_{+}^{d},\\
    % {2d \times 2d}\\
    &q'(\nu) =-\nu_1^\top \mathbf{1} + (1 + \frac{1}{\sigma}) \nu_2^\top \mathbf{1} + \frac{(\nu_4 + \nu_6)^\top \bf1}{\sigma} + \frac{\nu_{8}}{\sigma^2} \in \R. \\
\end{align*}
\end{proof}

\begin{align*}
      \OPT_2 \leq& \sup_{\substack{\bfzadv \in \mathcal{A}, \bftheta : \|\bftheta\|_2 \leq \frac{1}{\sigma}, \\ \bfq \in \R^N, \bfw \in \R^{dN}}} -\bmat{\bftheta \\ \bfzadv}^\top\bmat{[1 - (1 - \sigma \eta)^2] \bfA & -\epsilon (1 - \sigma \eta) \eta \bfA\\ -\epsilon (1 - \sigma \eta) \eta \bfA & -\epsilon \eta^2 \bfA} \bmat{\bftheta \\ \bfzadv} +
      \\& \bmat{-\sigma \eta \bfb \\ \epsilon \eta \bfb}^\top\bmat{\bftheta \\ \bfzadv} + \frac{1}{N} \sum_{1=1}^N \bmat{1 \\ (1 - \epsilon) \eta^2 \bfz_i^\top \bfA \bfz_i + (1 - \epsilon) \eta \bfb^\top \bfz_i \\ 2(1 - \epsilon) \eta (1 - \sigma \eta) \bfA \bfz_i}^\top\bmat{p_i \\ q_i \\ \bfw_i} \\
      &\text{such that} \\& \bfz_i^\top \bftheta + (1 + \frac{1}{\sigma}) q_i \geq 1 \; \forall i \in [N],
      \\& -\bfz_i^\top \bftheta -(1 + \frac{1}{\sigma}) q_i  \geq -(2 + \frac{1}{\sigma}) \; \forall i \in [N],
          \\& q_i \mathbf{1} / \sigma + \bfw_i \succeq \mathbf{0} \;\forall i \in [N], \\&
      -\bftheta - q_i \mathbf{1} / \sigma + \bfw_i \succeq  - \mathbf{1} / \sigma \;\forall i \in [N], \\&
       q_i \mathbf{1} / \sigma -\bfw_i \succeq \mathbf{0} \;\forall i \in [N], \\&
        \bftheta  - q_i \mathbf{1} / \sigma - \bfw_i \succeq -\mathbf{1} / \sigma\;\forall i \in [N],
      \\& \bfq \succeq \bf0,
      \\& -\bfq \succeq -\bf1,
      \\& \bfp \succeq \bf0,
      \\& -\bfp \succeq -\bf1,
      \\&1 - \bftheta^\top \bfzadv \geq 0,\\
      & \bftheta^\top \bfz_i + \frac{1}{\sigma} p_i \geq 0\\
      &- \bftheta^\top \bfz_i - \frac{1}{\sigma} p_i \geq -\frac{1}{\sigma}
      % \\&
      % -\bftheta \succeq -\mathbf{1} / \sigma,\\&
      % \bftheta \succeq -\mathbf{1} / \sigma,\\&
      % \bfzadv \in \mathcal{A}.
\end{align*}
\section{Experiment Details}\label{appendix:experiment}
\subsection{Vision Datasets}
We use the following pre-processing for all datasets. Use pretrained ResNet18 on ImageNET to extract features as the output of the pre-final layer. Perform Singular Value Decomposition (SVD) to select the top $d=30$ directions with largest feature variation. We normalise our dataset to ensure zero mean, append $1$ to each feature vector to capture the bias term, and scale to ensure each datapoint has $\ell_2$ norm less than or equal to 1. We multiply the $\{+1, -1\}$ labels to our features.

We split the training set into $\Pdata$ and $\Ptarget$ which are used to compute the certificates. For the simulation of the learning algorithms with various attacks, the same $\Pdata$ and $\Ptarget$ are used. The evaluation of these learning algorithms is done on a test set which is distributionally similar to $\Ptarget$.

\subsection{Reward Modeling}
We use the following pre-processing the HelpSteer dataset. We pass as inputs a concatenated text of the prompt, response pair to a pretrained BERT model to extract features as the output of the pre-final layer. We then perform Singular Value Decomposition (SVD) to select the top $d=30$ directions with largest feature variation. We normalise our dataset to ensure zero mean, append $1$ to each feature vector to capture the bias term, and scale to ensure each datapoint has $\ell_2$ norm less than or equal to 1.

The scores on each attribute in the raw dataset are natural numbers between $[0, 4]$. We linearly scale it to lie in the range $[-1, 1]$. We then multiply the labels to our features.

We split the training set into $\Pdata$ and $\Ptarget$ which are used to compute the certificates. For the simulation of the learning algorithms with various attacks, the same $\Pdata$ and $\Ptarget$ are used. The evaluation of these learning algorithms is done on a test set which is distributionally similar to $\Ptarget$.
\end{document}

% If your paper is accepted and the title of your paper is very long,
% the style will print as headings an error message. Use the following
% command to supply a shorter title of your paper so that it can be
% used as headings.
%
%\runningtitle{I use this title instead because the last one was very long}

% If your paper is accepted and the number of authors is large, the
% style will print as headings an error message. Use the following
% command to supply a shorter version of the authors names so that
% they can be used as headings (for example, use only the surnames)
%
%\runningauthor{Surname 1, Surname 2, Surname 3, ...., Surname n}

% Supplementary material: To improve readability, you must use a single-column format for the supplementary material.
\onecolumn
\aistatstitle{Test}

\section{FORMATTING INSTRUCTIONS}

To prepare a supplementary pdf file, we ask the authors to use \texttt{aistats2025.sty} as a style file and to follow the same formatting instructions as in the main paper.
The only difference is that the supplementary material must be in a \emph{single-column} format.
You can use \texttt{supplement.tex} in our starter pack as a starting point, or append the supplementary content to the main paper and split the final PDF into two separate files.

Note that reviewers are under no obligation to examine your supplementary material.

\section{MISSING PROOFS}

The supplementary materials may contain detailed proofs of the results that are missing in the main paper.

\subsection{Proof of Lemma 3}

\textit{In this section, we present the detailed proof of Lemma 3 and then [ ... ]}

\section{ADDITIONAL EXPERIMENTS}

If you have additional experimental results, you may include them in the supplementary materials.

\subsection{The Effect of Regularization Parameter}

\textit{Our algorithm depends on the regularization parameter $\lambda$. Figure 1 below illustrates the effect of this parameter on the performance of our algorithm. As we can see, [ ... ]}

\vfill